%% file: matroid.tex
\documentclass[11pt,letterpaper]{article} 

\usepackage{amsmath,amsfonts,amsthm,amssymb,dsfont}
\usepackage{fullpage}
\usepackage{comment}
\usepackage{soul,color}
\usepackage{graphicx,float,wrapfig}
\usepackage{mathrsfs}
\usepackage[usenames,dvipsnames]{pstricks}
\usepackage{pst-grad} 
\usepackage{pst-plot} 
\usepackage[linesnumbered,boxed,ruled,vlined]{algorithm2e}


\definecolor{Darkblue}{rgb}{0,0,0.4}
\definecolor{Brown}{cmyk}{0,0.61,1.,0.60}
\definecolor{Purple}{cmyk}{0.45,0.86,0,0}
\usepackage[backref,colorlinks,linkcolor=Brown,citecolor=blue,anchorcolor=Purple,bookmarks=true]{hyperref}

\newtheorem{prob}{Problem}
\newtheorem{theo}{Theorem}[section]

\newtheorem{lemma}[theo]{Lemma}

\newtheorem{prop}[theo]{Proposition}
\newtheorem{cor}[theo]{Corollary}
\newtheorem{defi}[theo]{Definition}
\newtheorem{rem}[theo]{Remark}

\newenvironment{proofof}[1]{\begin{proof}[Proof of #1]}{\end{proof}}

\newcommand{\topic}[1]{\vspace{0.2cm}\noindent{\bf {#1}:}}

\newcommand{\R}{\mathbb{R}}
\newcommand{\PR}{\mathbb{R}^+}
\newcommand{\Ex}{\mathbb{E}}
\newcommand{\betw}{\ |\ }

\newcommand{\distr}{\mathcal{D}}

\newcommand{\event}{\mathcal{E}}

\newcommand{\amean}[1]{\mu_{#1}}
\newcommand{\hamean}[1]{\hat{\mu}_{#1}}

\newcommand{\Bin}{\mathsf{Bin}}
\newcommand{\NegBer}{\mathsf{NegBin}}
\newcommand{\matroid}{\mathcal{M}}
\newcommand{\matband}{\mathcal{S}}
\newcommand{\indepfam}{\mathcal{I}}
\newcommand{\rank}{\mathrm{rank}}

\newcommand{\newcostfunc}[2]{\mu_{#1,#2}}

\newcommand{\UNIFORMSAMPLING}{\textsf{UniformSample}\xspace}
\newcommand{\NAIVE}{\textsf{Na\"ive-I}\xspace}
\newcommand{\NAIVETWO}{\textsf{Na\"ive-II}\xspace}
\newcommand{\bestkarm}{\textsc{Best-$k$-Arm}\xspace}
\newcommand{\bestarm}{\textsc{Best-$1$-Arm}\xspace}

\newcommand{\EPSMEANOPT}{average-$\epsilon$-optimal}

\newcommand{\Gap}[1]{\Delta_{[#1]}}

\newcommand{\OPT}{\mathsf{OPT}}
\newcommand{\OPTVAL}{\mathrm{val}}

\newcommand{\ELIMINATION}{\textsf{Elimination}}

\newcommand{\explorek}{\textsc{Explore}-$k$\xspace}

\newcommand{\matroidbandit}{\textsc{Best-Basis}\xspace}
\newcommand{\matroidbanditexact}{\textsc{Exact-Basis}\xspace}
\newcommand{\matroidbanditstrong}{\textsc{PAC-Basis}\xspace}
\newcommand{\matroidbanditAvg}{\textsc{PAC-Basis-Avg}\xspace}

\newcommand{\RECURPRUN}{\textsf{PAC-SamplePrune}\xspace}
\newcommand{\RECELIMI}{\textsf{AvgPAC-RecurElim}\xspace}
\newcommand{\EXPGAPMATROID}{\textsf{Exact-ExpGap}\xspace}

\newcommand{\eat}[1]{}

\newcommand{\pvalue}{0.01}

\newcommand{\jian}[1]{{\red Jian: #1}}
\newcommand{\lijie}[1]{{\magenta Lijie: #1}}

\newcommand{\initOneLiners}{%
    \setlength{\itemsep}{0pt}
    \setlength{\parsep }{0pt}
    \setlength{\topsep }{0pt}
}
\newenvironment{OneLiners}[1][\ensuremath{\bullet}]
    {\begin{list}
        {#1}
        {\initOneLiners}}
    {\end{list}}

\title{Pure Exploration of Multi-armed Bandit Under Matroid Constraints\thanks{Accepted for presentation at Conference on
		Learning Theory (COLT) 2016}}

\author{Lijie Chen\thanks{ Institute for Interdisciplinary Information Sciences (IIIS), Tsinghua University, Beijing, China.
	Research supported in part by the National Basic
	Research Program of China grants 2015CB358700,
	2011CBA00300, 2011CBA00301, and the National NSFC
	grants 61033001, 61361136003.
		}
	 \and Anupam Gupta\thanks{Department of Computer Science,
	 		Carnegie Mellon University,
	 		Pittsburgh, USA. Research partly supported by NSF awards CCF-1016799 and CCF-1319811.}	 	 
	 \and Jian Li\footnotemark[2]	 	
	 	}

\begin{document}
		\maketitle 	

\begin{abstract}
	We study the pure exploration problem subject to a matroid constraint	(\matroidbandit) in a stochastic multi-armed bandit game.
	In a \matroidbandit\ instance, we are given $n$ stochastic arms with unknown reward distributions, as well as a matroid $\matroid$ over the arms.
	Let the weight of an arm be the mean of its reward distribution.
	Our goal is to identify a basis of $\matroid$ with the maximum 
	total weight, using as few samples as possible.

	The problem is a significant generalization of the best arm 
	identification problem and the top-$k$ arm identification problem, which have attracted significant attentions in 
	recent years. 
	We study both the exact and PAC versions of \matroidbandit, and provide algorithms with nearly-optimal sample complexities for these versions.
	Our results generalize and/or improve on several previous 
	results for the top-$k$ arm identification problem and the 
	combinatorial pure exploration problem when the combinatorial
	constraint is a matroid.
\end{abstract}


    \input{intro}
    \input{prelim}
    \input{epsopt}
    \input{pureexploration}

\input{epsmean}

\section{Future Work}

In this paper, we present nearly-optimal algorithms for both the exact
and PAC versions of the pure-exploration problem subject to a matroid
constraint in a stochastic multi-armed bandit game: given a set of arms
with a matroid constraint on them, pick a basis of the matroid whose
weight (the sum of expectations over arms in this basis) is as large as
possible, with high probability.

An immediate direction for investiation is to extend our results to other
polynomial-time-computable combinatorial constraints: $s$-$t$ paths,
matchings (or more generally, the intersection of two matroids), etc.
The model also extends to NP-hard combinatorial constraints, but there
we would likely compare our solution against $\alpha$-approximate
solutions,
instead of the optimal solution. Considering non-linear functions of the
means is another natural next step. Yet another, perhaps more
challenging, direction is to consider stochastic optimization problems,
where the solution may depend on other details of the distributions than
just the means.

	

	\bibliographystyle{alpha}
	\bibliography{team}

\appendix

	\input{appendix}

\end{document}

%% file: intro.tex
	\section{Introduction}
	\label{sec:intro}
	
	The stochastic multi-armed bandit is a classical model for characterizing the exploration-exploitation tradeoff in many decision-making problems in stochastic environments.
	The popular objectives include maximizing the cumulative sum of rewards, or minimizing the cumulative regret
	(see e.g.,~\cite{cesa2006prediction,bubeck2012regret}). 
	However, in many application domains,
	the exploration phase and the evaluation phase are separated.
	The decision-maker can perform a \emph{pure-exploration phase}
	to identify an optimal (or nearly optimal) solution,
	and then keep exploiting this solution.
	Such problems arise in application domains
	such as medical trials \cite{robbins1985some,audibert2010best},
	communication network \cite{audibert2010best},
	crowdsourcing \cite{zhou2014optimal,cao2015top}.
	In particular, 
	the problem of identifying the single best arm
	in a stochastic bandit game has been
	has received considerable attention in 
	recent years~\cite{audibert2010best,even2006action,mannor2004sample,jamieson2014lil,karnin2013almost,chen2015optimal}.
	The generalization to identifying the top-$k$ arms
	has also been
	studied extensively~\cite{gabillon2012best,kalyanakrishnan2012pac,kaufmann2013information,kaufmann2014complexity,zhou2014optimal,cao2015top}.
	Since these problems are closely related to the problem 
	we study in the paper, we formally define it as follows.
	
	\begin{prob}
		(\bestkarm)
		There are $n$ unknown distributions $\distr_1,\distr_2,\dotsc,\distr_n$, all supported on $[0,1]$. 
		Let the mean of $\distr_i$ be $\mu_i$. 
		At each step we choose a distribution and get an i.i.d.\ sample from the distribution. 
		Our goal is to find the $k$ distributions with the largest means (exactly or approximately), 
		with probability at least $1-\delta$, using as few samples as possible.
	\end{prob}
        The distributions above are also called {\em arms} in the multi-armed bandit literature. 	
	We denote the $k^{th}$ largest mean by $\mu_{[k]}$.
	In addition, we assume $\mu_{[k]}$ and $\mu_{[k+1]}$ are different (so the optimal top-$k$ answer is unique).

	In certain applications such as online ad allocations,
	there is a natural combinatorial constraint over the set of arms,
	and	we can only choose a subset of arms subject to the given
	constraint (\bestkarm\ simply involves a cardinality constraint).
	Motivated by such applications, 
	Chen et al.~\cite{chen2014combinatorial}
	introduced the {\em combinatrial pure exploration} problem.
	They considered the general setting with arbitrary combinatorial constraint, and propose several algorithms. 
	In this paper, we consider the same problem under a {\em matroid} constraint,
	one	of the most popular combinatorial constraint. 
	The matroid constraint was also discuss in length
	in \cite{chen2014combinatorial}.	
	
	The notion of matroid (see Section~\ref{sec:prel}
	for the definition) is an abstraction of many combinatorial
	structures, including the sets of linearly independent vectors
	in a given set of vectors, the sets of spanning forests in an 
	undirected graph and many others. 
	We note that \bestkarm\ is a special case
	of a matroid constraint, since all subsets of size of at most $k$ form a {\em uniform matroid}.  	 
	Now, we formally define the {\em matroid 
	pure exploration bandit problem} as follows.
	
	\begin{defi}\label{defi:matroid-bandit-instance}
		(\matroidbandit)
		In a \matroidbandit\ instance $\matband=(S,\matroid)$,
		we are given a set $S$ of $n$ arms.
		Each arm $a \in S$ is associated with 
		an unknown reward distribution $\distr_a$, supported on $[0,1]$, with mean $\mu_a$ (which is unknown as well). 
		Without loss of generality, we assume all arms have distinct means.
		
		We are also given a matroid $\matroid=(S,\indepfam)$ with ground set identified with the set $S$ of arms.
		The weight function 
		$\mu: S \to \PR$ simply sets the weight of $a$ to be the
                mean of $\distr_a$; i.e., 
		$\mu(a)=\amean{a}$ for all $a\in S$. The weights are
                initially unknown, and are only learned by sampling arms.
		Our goal is to find a basis (a.k.a.\ a maximal independent set) 
		of the matroid with 
		the maximum total weight/cost (exactly or approximately), 
		with probability at least $1-\delta$, 
		using as few samples as possible.
	\end{defi}
	
	Besides including \bestkarm as a special case, the
        \matroidbandit\ problem also captures the following natural
        problems, motivated by various applications.
        \begin{enumerate}
        \item Suppose we have $m$ disjoint groups $G_1,\ldots, G_m$ of
          arms, and we would like to pick the best $k_i$ arms from group
          $G_i$ (where $k_i$s are given integers). This is exactly the
          best-basis problem for a \emph{partition matroid}.  Note that
          PAC version of the problem cannot be modeled as a disjoint collection
          of best-k-problems.

          The special case where $k_i=1$ has been studied in
          \cite{gabillon2011multi,bubeck2012multiple} (under the fixed
          budget setting). They are motivated by a clinical problem with
          $m$ subpopulations, where one would like to decide the best
          $k_i$ treatments from the options available for subjects from
          each subpopulation.
			
        \item Beside the above constraints for the groups, we may have
          an additional global constraint on the total number of arms we
          can choose. This is a special case of a \emph{laminar matroid}.
			
        \item An application mentioned in \cite{chen2014combinatorial}
          is the following. Consider a network where the delay of the
          links are stochastic.  A network routing system wants to build
          a minimum spanning tree to connect all nodes, where the weight
          of each edges are expected delay of that link.  A spanning
          tree is a basis in a \emph{graphical matroid}.
			
        \item Consider a set of workers and a set of tasks. Each worker
          is able to do only a subset of tasks (which defines a
          worker-task bipartite graph). Each worker must be assigned to
          one task (so we need to build a matching between the workers
          and the tasks) and the reward of a task is stochastic. We
          would like to identify the set of tasks that can be completed
          by the set of workers and have maximum total reward. This
          combinatorial structure (over the subsets of tasks) is a
          \emph{transversal matroid}. This problem (or variants) may find
          applications in crowdsourcing or online advertisement.
        \end{enumerate}

	There are two natural formulations of the \matroidbandit\ problem: in one, we need to identify the unique
	optimal basis with a certain confidence, and in some others we can settle for an approximate optimal basis (the PAC setting).
	We now formally define these problems, and present our results.
	


	\subsection{Identifying the Exact Optimal basis}
        \label{sec:exact-basis}
	
	\begin{defi}(\matroidbanditexact) Given a \matroidbandit\
          instance $\matband=(S,\matroid)$ and a confidence level
          $\delta > 0$, the goal is to output the optimal basis of
          $\matroid$ (one that maximizes $\sum_{a\in I}\mu_a$)
          with probability at least $1-\delta$, using as few samples as possible.
	\end{defi}
	Without loss of generality,
	assume that matroid $\matroid$ has no isolated elements
	(i.e., elements that are included in every basis) and no loops
        (i.e., elements that belong to no basis),
	since we can always include or ignore them without affecting the
        solution.
	We use $\OPT(\matroid)$ to denote the optimal basis 
	(as well as the optimal total weight) for matroid $\matroid$.
	For a subset of elements $F\subseteq S$, 
	let $\matroid_F$ denote the restriction of $\matroid$
	to $F$, 
	and 
	$\matroid_{/F}$ denote the contraction of $\matroid$
	by $F$ (see Definition~\ref{defi:matroid-rescon}).
	Note that $\OPT(\matroid_{/\{e\}}) + \mu(e)$ is the optimal cost
	among all bases including $e$.

	Naturally, the sample complexity of an algorithm for
        \matroidbanditexact\ depends on the parameters of the problem
        instance. In particular, we need to define the following gap
        parameter.

        \begin{defi}[Gap]
	\label{def:gap}
	Given a matroid $\matroid=(S,\indepfam)$ with cost function
        $\mu: S \to \PR$, such that all costs are distinct, define the
        \emph{gap} of an element $e \in S$ to be 
	$$
	\Delta_{e}^{\matroid,\mu} := \begin{cases}
	\OPT(\matroid) - \OPT(\matroid_{S \setminus \{e\}})
	\quad & e \in \OPT(\matroid)  \\
	\OPT(\matroid) - ( \OPT(\matroid_{/\{e\}}) + \mu(e) )
	\quad &  e \not\in \OPT(\matroid)\\
	\end{cases}
	$$
\end{defi}

Intuitively, for an element $e\in \OPT(\matroid)$, 
its gap is the loss if we do not select $e$,
whereas for an element $e\notin \OPT(\matroid)$, 
its gap is the loss if we are forced to select $e$.
Since we assume that elements have distinct weights, 
 $\Delta_{e} > 0$ for all arms $e$. 
We note that Definition~\ref{def:gap} 
is the same as the gap definition in \cite{chen2014combinatorial} 
and
generalizes
the gaps defined for the \bestkarm\ problem used in  \cite{kalyanakrishnan2012pac}
(in \bestkarm, the gap of an arm $e$ to be
	$
	\Delta_{e}=\mu_e -\mu_{[k+1]}
	$ if $e$ is a top-$k$ arm,
	and	
	$
	\Delta_{e}=\mu_{[k]} -\mu_{e}
	$
	otherwise).

Chen et al.~\cite{chen2014combinatorial} obtained an algorithm
with sample complexity 
$$
\left(\sum_{e \in S} \Delta_e^{-2}(\ln\delta^{-1}+\ln n + \ln\sum\nolimits_{e \in S} \Delta_e^{-2} )\right),
$$
when specialized to \matroidbanditexact. \footnote{Their algorithm works
	for arbitrary combinatorial constraint.
	The sample complexity depends on a {\em width} parameter of 
	the constraint, which is roughly the number of elements
	needed to be exchanged from one feasible solution to another.
	The width can be as large as $n$.
	For a matroid, the width is 2.
	}
We improve upon their result by proving the following theorem.

\begin{theo}[Main Result for Exact Identification]
  \label{theo:pure-exploration}
	There is an algorithm for \matroidbanditexact, that returns the optimal basis for $\matband$, with probability at least $1-\delta$, and uses
	at most 
	$$
	O\left(\sum_{e \in S} \Delta_{e}^{-2}(\ln \delta^{-1}+\ln k 
	+ \ln\ln \Delta_e^{-1}) \right)
	$$
	samples. Here, $k = \rank(\matroid)$ is the size of a basis of $\matroid$.
\end{theo}
Observe that the dependence is now on the rank of the matroid $k$,
rather than the number of elements $n$ which may be much larger than $k$. Moreover, the dependence on
$\Delta_e$ is doubly logarithmic. 


For the special case of the $k$-uniform matroid, the problem becomes the 
 \bestkarm\ problem, for which the current state-of-the-art is
$O(\sum_{i=1}^{n} \Gap{i}^{-2}(\ln\delta^{-1}+\ln\sum\nolimits_{i=1}^{n} \Gap{i}^{-2} )),$
obtained by \cite{kalyanakrishnan2012pac}.
Theorem~\ref{theo:pure-exploration} improves 
upon this result for the typical case when $\ln\sum\nolimits_{i=1}^{n} \Gap{i}^{-2}$ 
 is larger than $\ln k$.
Theorem~\ref{theo:pure-exploration} also matches
the recent upper bound of
$O(\sum\nolimits_{i=2}^{n} \Gap{i}^{-2} (\ln\ln\Gap{i}^{-1}+\ln\delta^{-1}))$
for \bestarm,
due to Karnin et al.~\cite{karnin2013almost}
and	Jamieson et al.~\cite{jamieson2014lil}.

Chen et al.~\cite{chen2014combinatorial} proved an $\Omega(\sum_{e \in S} \Delta_{e}^{-2} \ln\delta^{-1})$ lower bound for the problem.
Moreover, Kalyanakrishnan et al.~\cite{kalyanakrishnan2012pac} 
showed an
$\Omega(n\varepsilon^{-2}(\ln\delta^{-1}+\ln k))$ lower bound for a PAC version (the \explorek\ metric, see Section~\ref{sec:pac}) 
of \bestkarm. 
Indeed, in their lower bound instances, all arms have gap $\Delta_{e} = \varepsilon$.
If we apply our exact algorithm on those instances,
the sample complexity is 
$O(n\varepsilon^{-2}(\ln\delta^{-1}+\ln k+\ln\ln \varepsilon^{-1}))$.
Hence, the first two terms of our upper bound are probably 
necessary in light of the above lower bounds.

\subsection{The PAC setting}
\label{sec:pac}

Next we discuss our results for the PAC setting.
Several notions of approximation were used for the
special case of \bestkarm, when we return a set $I$ of $k$ arms.
Kalyanakrishnan et al.~\cite{kalyanakrishnan2012pac} required that
the mean of every arm in $I$ be at 
least $\mu_{[k]}-\varepsilon$ (The \explorek\ metric).
Zhou et al.~\cite{zhou2014optimal} 
required that the average mean 
$\frac{1}{k}\sum_{e\in I}\mu_{e}$
of $I$ be at least $\frac{1}{k}\sum_{i=1}^{k}\mu_{[i]}-\varepsilon$; we
call such a solution an
{\em average-$\varepsilon$-optimal} solution.
Finally, Cao et al.~\cite{cao2015top} proposed a stronger 
metric that required the mean of the $i^{th}$ arm in $I$
be at least $\mu_{[i]}-\varepsilon$, for all $i\in [k]$.
This notion, which we call {\em elementwise-$\varepsilon$-optimality}
extends to general matroids: we need that 
 $i^{th}$ largest arm in our solution is 
at least 
the $i^{th}$ largest mean in the optimal solution minus~$\varepsilon$.

In this paper we introduce the stronger notion of an {\em $\varepsilon$-optimal} solution.

\begin{defi}\label{defi:eps-opt}
	(\matroidbanditstrong and $\varepsilon$-optimality)
	We are given a matroid $\matroid=(S,\indepfam)$
	with cost function $\mu : S \to \PR$. 
	We say a basis $I$ is \emph{$\varepsilon$-optimal} (with respect to $\mu$), if $I$ is an optimal solution 
	for the modified cost function $\newcostfunc{I}{\varepsilon}$, defined as follows:
	\[
	\newcostfunc{I}{\varepsilon}(e) = \begin{cases}
	\mu(e) + \varepsilon\quad & \quad \text{for }e \in I\\
	\mu(e)  \quad & \quad \text{for }e \not\in I. 
	\end{cases}
	\]
	In other words, if we add $\varepsilon$ to each element in $I$, $I$
        would become an optimal solution.
\end{defi}



The proof of the following proposition can be found
in the appendix.

\begin{prop}
\label{prop:epsoptimality}
For a \matroidbandit\ instance,
an $\varepsilon$-optimal solution is also elementwise-$\varepsilon$-optimal.
The converse is not necessarily true.
\end{prop}


\begin{theo}[Main Result for PAC Setting]
  \label{theo:eps-opt-algo}
	There is an algorithm for
	\matroidbanditstrong\ which
	returns an $\varepsilon$-optimal solution for $\matband=(S,\matroid)$,
	with probability at least $1-\delta$, and uses at most
	$$
	O(n\varepsilon^{-2} \cdot (\ln k + \ln \delta^{-1})) 
	$$	
	samples, where $k = \rank(\matroid)$. 
\end{theo}

This theorem generalizes and strengthens  
the results in \cite{kalyanakrishnan2012pac,cao2015top},
in which the same sample complexity was obtained 
for \bestkarm\ under \explorek\ and 
elementwise-$\varepsilon$-optimality metrics, respectively.
In fact, this sample complexity is optimal, since an $\Omega(n\varepsilon^{-2}(\ln k
+ \ln\delta^{-1}))$ lower bound is known for \explorek\ for the special
case of \bestkarm, due to~\cite{kalyanakrishnan2012pac}.

\subsubsection{Average-$\varepsilon$-optimality}

We also consider the weaker notion of average-$\varepsilon$-optimality,
which may suffice for certain applications. For this definition, we give
another algorithm 
with a lower sample complexity.

\begin{defi}\label{defi:matroid-mean-eps-opt}
	(\matroidbanditAvg)
	Given a matroid $\matroid=(S,\indepfam)$
	with cost function $\mu : S \to \PR$. 
	Suppose $k = \rank(\matroid)$.
	We say a basis $I$ is an \EPSMEANOPT\ solution  (w.r.t.\ $\mu$), if:
	$\frac{1}{k}\sum_{e\in I} \mu(e) \ge 
	\frac{1}{k}\OPT(\matroid) - \varepsilon.$     
\end{defi}


\begin{theo}\label{theo:eps-mean-algo}
	There is an algorithm for \matroidbanditAvg, 
	which can return an \EPSMEANOPT\ solution for $\matband$,
	with probability at least $1-\delta$, and its sample complexity is at most
	\[
	O\Big( \big(n\cdot(1 + \ln\delta^{-1}/k) + (\ln\delta^{-1}+k)(\ln k\ln\ln k + \ln\delta^{-1}\ln\ln\delta^{-1} )\big)\varepsilon^{-2} \Big).
	\]
	In particular, when
	$k\ln\delta^{-1} \le O(n^{0.99})$ and $\delta \ge \Omega(\exp(-n^{0.49}))$,
	the sample complexity is 
	$$
		O( n\varepsilon^{-2}(1 + \ln\delta^{-1}/k)).
	$$
\end{theo}


\cite{zhou2014optimal} obtained  
matching upper and lower bounds of
	$
	\Omega( n\varepsilon^{-2}(1 + \ln\delta^{-1}/k))
	$
for \bestkarm\ under the average metric.
Our result matches their result 
when $\delta$ is not extremely small and 
$k$ is not very close to $n$.
Obtaining tight upper and lower bounds for 
all range of parameters is left as an interesting open question.

\subsubsection{Prior and Our Techniques}

Several prior algorithms for the PAC versions of \bestarm\ and
\bestkarm\ (e.g., \cite{karnin2013almost,zhou2014optimal,even2002pac}) were elimination-based,
roughly using the following framework: In the $r^{th}$ round, we sample
each remaining arm $Q_r$ times, \footnote{Typically, $Q_r$ increases
  exponentially with $r$.} and eliminate all arms whose empirical means
fall below a certain threshold. This threshold can be either a
percentile, as in~\cite{even2002pac,zhou2014optimal} or an
$\varepsilon$-optimal arm obtained by some PAC algorithm, such as
in~\cite{karnin2013almost}.  After eliminating some arms, we proceed to
the next round.  Small variations to this procedure are possible, e.g.,
if the number of remaining arms is not much larger than $k$, we can
directly use the na\"ive uniform sampling algorithm.  A main difference
in prior works is in their analysis, due to the different PAC-optimality
metrics.  However, we cannot easily extend this framework to either
\matroidbanditstrong\ or \matroidbanditAvg, since it is not clear how to
eliminate even a small constant fraction of arms while ensuring that the
optimal value for the remaining set does not drop. Indeed, due to the
combinatorial structure of the matroid, we cannot perform elimination
based solely on fixed thresholds.
 
We resolve the issue by applying a sampling-and-pruning technique
developed by Karger, and used by Karger, Klein, and Tarjan in their
expected linear-time randomized algorithm for minimum spanning tree. 
Here is the high-level idea, in the context of
the \matroidbanditstrong problem.
We pick a random subset $F$ by including each arm independently with
some small constant probability $p$, and recursively find an
$\varepsilon/3$-optimal basis $I$ for the subset $F$.  The key idea is
that this basis $I$ can be used to eliminate a significant proportion of
arms, while ensuring that the remaining set still contains a desirable
solution.  Hence, after eliminating those arms, we can recurse on the
remaining arms.  Unlike the previous algorithms which eliminate arms
based on a single threshold, we perform the elimination based on the
solution $I$ of a random subset.  We feel this extension of the sampling
and pruning technique to bandit problems will find other applications.

Another popular approach for pure exploration problems is based on upper
or lower confidence bounds (UCB or LUCB) (see e.g., \cite{kalyanakrishnan2012pac,chen2014combinatorial}). While being very flexible and easy to apply, the analysis of all
such bounds inevitably requires a union bound of all rounds (which is at
least $n$), thus incurring a $\log n$ factor, which is worse than the
optimal $\log k$ factor that we obtain.

\subsection{Other Related Work}



The problem of identifying the single best arm,
a very special case of our problem,
has been studied extensively.
For the PAC version of the problem,
\footnote{
	Since the solution only contains one arm,
	all different notions of PAC optimality mentioned in
	Section~\ref{sec:pac} are equivalent.  
}
Even-Dar et al.~\cite{even2002pac} provided an algorithm with  sample 
complexity
$O(n\varepsilon^{-2} \cdot \ln \delta^{-1})$,
which is also optimal.
For the exact version, 
Mannor and Tsitsiklis~\cite{mannor2004sample} proved
a lower bound
of $\Omega(\sum\nolimits_{i=2}^{n} \Gap{i}^{-2} \ln\delta^{-1})$.
\cite{farrell1964asymptotic} showed a
lower bound of
$\Omega(\Gap{2}^{-2}\ln\ln \Gap{2}^{-1})$
even if there are only two arms.
Karnin et al.~\cite{karnin2013almost}
obtained an upper bound of 
$O(\sum\nolimits_{i=2}^{n} \Gap{i}^{-2} (\ln\ln\Gap{i}^{-1}+\ln\delta^{-1}))$,
matching Farrell's lower bound for two arms.
Jamieson et al.~\cite{jamieson2014lil} obtained the same result using a UCB-like algorithm.
Very Recently, 
Chen and Li~\cite{chen2015optimal} 
provided a new lower bound of 
$\Omega(\sum_{i=2}^{n} \Gap{i}^{-2} \ln\ln n)$
and an improved upper bound
of $
O\Big(
\Gap{2}^{-2}\ln\ln \Gap{2}^{-1}+
\sum_{i=2}^{n} \Gap{i}^{-2} \ln\delta^{-1}+\sum_{i=2}^{n} \Gap{i}^{-2}\ln\ln \min(n,\Gap{i}^{-1}) 
\Big).
$

In all aforementioned results, we require that the (PAC or exact) algorithm returns a correct answer with probability 
at least $1-\delta$. 
This is called the {\em fixed confidence} setting
in the literature.
Another popular setting is the {\em fixed budget} setting, in which
the total number of samples is subject to a given budget constraint,
and we would like to minimize the failure probability
(see e.g., \cite{bubeck2013multiple,gabillon2012best,karnin2013almost,chen2014combinatorial}).
Some prior work
(\cite{audibert2010best,bubeck2013multiple,audibert2013regret})
also considered the objective of making the {\em expected simple regret} at most $\varepsilon$ 
(i.e., $\frac{1}{k}(\sum_{i=1}^k \mu_{[i]} -\mathbb{E}[\sum_{a\in T}\mu_a])\leq \varepsilon$),
which is a somewhat weaker objective.


There is a large body of work on minimizing the cumulative regret
in online multi-armed bandit games with various combinatorial constraints
in different feedback settings
(see e.g., \cite{cesa2006prediction,bubeck2012regret,cesa2012combinatorial,audibert2013regret,chen2013combinatorial} and the references therein).
In an online bandit game, there are $T$ rounds.
In the $t^{th}$ round, we can play a combinatorial subset $S_i$ of arms.
The goal is to minimize 
$T\sum_{a\in \OPT} \mu_a - \sum_{t=1}^T \sum_{a\in S_t} \mu_a$. 
We note that it is possible to obtain an expected simple regret of 
$\varepsilon$ for \matroidbandit, 
with at most $O(n\varepsilon^{-2})$ samples, using the semi-bandit regret bound in
\cite{audibert2013regret}.
In particular, 
they provided an online mirror descent algorithm and showed a cumulative regret of $\sqrt{knT}$ in the semi-bandit feedback setting
(i.e., we can only observe the rewards from the arms we played
),
where $k$ is the maximum cardinality of a feasible set. By setting $T=nk^{-1}\varepsilon^{-2}$, we get a cumulative 
regret of $n/\varepsilon$. 
If we uniformly randomly pick a solution from $\{S_t\}_{t\in [T]}$, 
we can see  that
$\mathbb{E}_{t\in [T]}\frac{1}{k}(\sum_{a\in \OPT} \mu_a - \sum_{a\in S_t} \mu_a)\leq \varepsilon$. One drawback of their algorithm is that it needs to solve a convex program over the matroid polytope, which can be computationally expensive, while our algorithm is purely combinatorial and very easy to implement.

In recent and concurrent work, Gabillon et
al.~\cite{gabillon2016improved} proposed a new complexity notion for the
general combinatorial pure exploration problem, and developed new
algorithms in both fixed budget and the fixed confidence setting. They
showed that in some cases, the sample complexity of their algorithms is
better than that of \cite{chen2014combinatorial}. While the current
implementations of their algorithm have an exponential running time,
even for general matroid constraints, it is an interesting problem to
get more efficient algorithms, and to combine their notion of complexity
with our techniques.


%% file: prelim.tex

\section{Preliminaries}
\label{sec:prel}

    



\subsection{Useful Facts about Matroids}


While there are many equivalent definitions for matroids, we  find this
one most convenient.

\begin{defi}[Matroid]\label{defi:matroid}
	A matroid $\matroid(S,\indepfam)$ consists of a finite set $S$
        (called the \emph{ground set}), and a non-empty family
        $\indepfam$ of subsets of $S$ (with sets in $\indepfam$ being called \emph{independent sets}), satisfying the following:
	\begin{OneLiners}
        \item[i.] Any subset of an independent set is an independent set.
        \item[ii.] Given two sets $I,J \in \indepfam$, if $|I| > |J|$, there exists  element 
          $e \in I \setminus J$ such that 
          $J \cup \{e\} \in \indepfam$.
        \end{OneLiners}
      \end{defi}

For convenience, we often write $I \in \matroid$ 
instead of $I\in \indepfam$
to denote that $I$ is an independent set of $\matroid$. 
An independent set is \emph{maximal} if it is not a proper subset of
another independent set; a maximal independent set is called a \emph{basis}. 


\begin{defi}[Rank]
	\label{defi:rank}
	Given matroid $\matroid(S,\indepfam)$ and set $A \subseteq S$,
        the \emph{rank} of $A$, denoted by $\rank_{\matroid}(A)$, is the
        cardinality of a maximal independent subset contained in $A$.
\end{defi}

When $\matroid$ is clear from context, we merely write $\rank(A)$. All
bases of a matroid have the same cardinality. We use $\rank(\matroid)$,
instead of $\rank(S)$, to denote the cardinality of every basis of $\matroid$.

\eat{
The spanning forests in a graph, and independent sets in a set of vectors are two well-known examples of matroid. Intuitively, matroid captures the good combinatorial structure such that the greedy algorithm can work.
}

We often need to work with the set of independent sets
restricted to a subset of elements. 
Sometimes we can determine to include some elements as
a partial solution, 
we need to work the the rest of the matroid,
conditioning on the partial solution. 
We need the definitions of 
matroid restrictions and matroid contractions
to formalize the above situations.

\begin{defi}[Matroid restrictions and contractions] \label{defi:matroid-rescon} 
	Let $\matroid(S,\indepfam)$ be a matroid. 
	For $A\subseteq S$, we define the restriction of $\matroid$ to $A$ as follows: $\matroid_{A}$ is also a matroid with  ground set $A$; an independent set of $\matroid$ which is a subset of $A$ is an independent of $\matroid_A$.
	
	The contraction of $A$ is defined as follows: $\matroid_{/A}$ is the matroid with ground set 
	$S'=\{ e \in S \betw \rank(\{e\} \cup A) > \rank(A)  \}$, and the independent set family 
	$\indepfam' = \{ I \subseteq S'
	\mid \rank(I \cup A) = |I| + \rank(A) \}$.
\end{defi}

Both $\matroid_A$ and $\matroid_{/A}$ are indeed matroids. 
Sometimes, we may also write $\matroid|A$ and $\matroid/A$ 
to avoid successive subscripts.
In our paper, we only need to contract an independent set 
$A\in \indepfam$.
In this case, $\rank(A)=|A|$,
and the definition simplifies to the following:
a set $I$ (disjoint from $A$) 
is independent in $\matroid_{/A}$,
if $I\cup A$ is independent in $\matroid$.

\begin{defi}[Isolated Elements and Loops]
\label{defi:iso-loop}
	For a matroid $\matroid=(S,\indepfam)$ and element $e \in S$, 
	we say $e$ is an {\em isolated element}, 
	if it is contained in all bases of $\matroid$
	(or equivalently, $\rank(S) > \rank(S \setminus \{e\})$). We say 
	$e$ is a {\em loop}
	if it belongs to no basis of $\matroid$.
\end{defi}


Clearly, since the mean of each arm is nonnegative,
we can directly select all isolated elements and contract 
out these elements. Also, we can simply ignore those loops.
From now on, we can assume without loss of generality that 
there is no isolated element or loop in $\matroid$. 

\begin{defi}[Block]
	Let $\matroid(S,\indepfam)$ be a matroid. 
	Given a subset $A \subset S$ and an element $e$ such that $e \not\in A$, we say $A$ blocks $e$, 
	if $\rank_\matroid(A\cup\{e\})=\rank_\matroid(A)$. 
\end{defi}

Intuitively, $e$ is blocked by $A$
if adding $e$ is not useful in increasing the cardinality of 
the maximal independent set in $A$.
Note that, if $A \subseteq B$, $e \not\in B$ and $A$ blocks $e$, then clearly $B$ also blocks $e$, 
due to the submodularity of $\rank$:
$\rank(B\cap \{e\})-\rank(B)\leq 
\rank(A\cap \{e\})-\rank(A)$. 

We have the following lemma characterizing 
when a subset $A$ blocks an element $e$.

\begin{lemma}\label{lm:blocks-1}
	If $A$ blocks $e$, every basis $I$ of $\matroid_{A}$ blocks $e$. 
\end{lemma}
\begin{proof}
	Since $A$ blocks $e$, $\rank_\matroid(A\cup\{e\})=\rank_\matroid(A)$.
	Consider a basis $I$ of $\matroid_A$.
	$\rank_\matroid(I\cup\{e\})\leq 
	\rank_\matroid(A\cup\{e\})=\rank_\matroid(A)=\rank_\matroid(I).$
	Hence, $I$ blocks $e$ as well.
	\eat{
	Let $J \in \matroid_A$ such that $J \cup \{e\} \not\in \matroid$ (i.e., $J$ blocks $e$). 
	Let $J'\in \matroid_A$ be a basis containing $J$.
	Clearly, $J' \cup \{e\} \not\in \matroid$, 
	by the first property of Definition~\ref{defi:matroid}.
	
	Let $I$ be an arbitrary basis in $\matroid_A$.
	We know that $|I| = |J'|$. 
	Suppose by contradiction that 
	$I \cup \{e\} \in \matroid$. 
	By the second property of Definition~\ref{defi:matroid},
	we know that there exists an element 
	$x \in I \cup \{e\} \setminus J'$, 
	such that $J' \cup \{x\} \in \matroid$,
	Since $J'$ is maximal in $\matroid_A$, 
	we can see that $x \not\in A$. 
	So it must be the case that $x=e$.
	But $J' \cup \{e\} \in \matroid$ (as $J'$ blocks $e$), rendering a contradiction. 
	So $I$ also blocks $e$.
	}
\end{proof}

Then we define what is an optimal solution for a matroid with respect to a cost function $\mu$.

\begin{defi}\label{defi:matroid-opt}
	Given a matroid $\matroid(S,\indepfam)$, and 
	an injective cost/weight function $\mu : S \to \PR$, let
	$\mu(I) := \sum\nolimits_{e \in I} \mu(e)$ denote the total
        weight of elements in the independent set $I \in M$.
	We say $I$ is an optimal basis (with respect to $\mu$) 
	if $\mu(I)$ has the maximum value among all independent sets
	in $\indepfam$. 
	We define $\OPT_{\mu}(\matroid) = \max_{I\in \indepfam}\mu(I)$.  
	With slight abuse of notation, we may 
	also use $\OPT_{\mu}(\matroid)$ to denote the optimal 
	basis. 
	When $\mu$ is clear from the context, 
	we simply write $\OPT(\matroid)$.
\end{defi}

From now on, we assume the cost of each element is distinct.
It is well known that the optimal basis $\OPT(\matroid)$
is unique (under the distinctness assumption) 
and can be obtained by a simple greedy algorithm:
We first sort the elements in the decreasing order 
of their cost. Then, we attempt to add the elements greedily 
one by one in this order, to the current solution, which is initially empty.

We are given matroid $\matroid=(S,\indepfam)$ with cost function $\mu: S \to \PR$. For a subset $A \subseteq S$, we define 
$$
A_{\mu}^{\ge a} := \{ e \in A \betw \mu(e) \ge a \}.
$$ 
We define $A_{\mu}^{> a},A_{\mu}^{\le a},A_{\mu}^{< a}$ similarly.
Sometimes we omit the subscript $\mu$ if it is clear from the context.
Finally, the following characterizations of optimal solutions for $\matroid$ all follow
from the greedy procedure.

\begin{lemma}\label{lm:char-opt}
	For a matroid $\matroid(S,\indepfam)$, cost function $\mu : S \to \PR$
	and basis $I\in \indepfam$,
	the following statements are equivalent:
		\begin{OneLiners}
		\item[i.] $I$ is an optimal basis for $\matroid$ with 
		respect to $\mu$.
		\item[ii.]
		For any  $e\in I$, $S^{> \mu(e)}$ does not block $e$.
		\item[iii.]
		For any $e \in S \setminus I$, $I^{\ge \mu(e)}$ blocks $e$.
		\item[iv.]
		For any $r \in \R$, $I^{\ge r}$ is a basis in $\matroid_{S^{\ge r}}$.
              \end{OneLiners}
	
\end{lemma}

\subsection{Uniform Sampling}

    The following na\"{\i}ve uniform sampling procedure
    will be used frequently.
    
    
\RestyleAlgo{boxed,ruled}
\SetAlgoVlined

    \begin{algorithm}[H]
    	
\LinesNumbered
\setcounter{AlgoLine}{0}
    	\caption{\UNIFORMSAMPLING($S,\varepsilon,\delta$)}
    	\label{algo:UNIFORM-SAMPLE-PROCEDURE}

    	\KwData{Arm set $S$, error bound $\varepsilon$, confidence level $\delta$.}
    	\KwResult{For each arm $a$, output the empirical mean
          $\hamean{a}$.}
        \smallskip
    	For each arm $a \in S$, sample it $\varepsilon^{-2}\ln(2\cdot
        \delta^{-1})/2$ times. Let $\hamean{a}$ be the empirical mean.

    \end{algorithm}
    
    The following lemma for Algorithm~\ref{algo:UNIFORM-SAMPLE-PROCEDURE},
    is an immediate consequence of Proposition~\ref{prop:chernoff}.1.
    
    \begin{lemma}\label{lm:UNIFORM-SAMPLE-PROCEDURE}
    	For each arm $a \in S$, we have that
    	$
    	\Pr\left[|\amean{a}-\hamean{a}| \ge \varepsilon\right] \le \delta.
    	$
    \end{lemma}


%% file: epsopt.tex
\section{An Optimal PAC Algorithm for the \matroidbanditstrong Problem}
\label{sec:matroid-PAC}

In this section, we
prove Theorem~\ref{theo:eps-opt-algo}
by presenting an algorithm
for \matroidbanditstrong\ with optimal sample complexity.
The algorithm is also a useful subprocedure
for both \matroidbanditexact\ and \matroidbanditAvg.

\subsection{Notation}

We first introduce 
an analogue of Lemma~\ref{lm:char-opt}
for $\varepsilon$-optimal solutions.

\begin{lemma}\label{lm:eps-opt-char}
	For a matroid $\matroid=(S,\indepfam)$ with 
	 cost function $\mu:S \to \PR$, and a basis $I$, 
	the following statements are equivalent:
	\begin{OneLiners}
		\item[1.] $I$ is  $\varepsilon$-optimal for $\matroid$ with
                  respect to $\mu$. 
		\item[2.] For any $e \in S \setminus I$, $I^{\ge \mu(e)-\varepsilon}$ blocks $e$.
		\item[3.]
 For any $r \in \R$, let $D_r = (S\setminus I)^{\ge r+\varepsilon} \cup I^{\ge r}$.
		$I^{\ge r}$ is a basis in $\matroid_{D_r}$.
              \end{OneLiners}
            \end{lemma}
\begin{proof}
	Apply Lemma~\ref{lm:char-opt} with the cost function $\mu_{I,\varepsilon}$,
	as defined in Definition~\ref{defi:eps-opt}.
\end{proof}

\begin{defi}[$\varepsilon$-Approximation Subset]
  \label{defi:eps-approx-subset}
	 Given a matroid $\matroid=(S,\indepfam)$ and cost function
         $\mu:S \to \PR$, 
	 let $A \subseteq B$ be two subsets of $S$. We say $A$ is an
         $\varepsilon$-approximate subset of $B$
	 if there exists an independent set $I\in\matroid_A$ 
	 such that $I$ is $\varepsilon$-optimal for $\matroid_B$ 
	 with respect to the cost function $\mu$.
	 
\end{defi}

\begin{lemma}
	\label{lm:approxsubset}
	Suppose $A$ is an $\varepsilon$-approximate subset of $B$,
	and $I\in\matroid_A$ 
	is $\varepsilon$-optimal for $\matroid_B$.
	For any $e \in B \setminus A$, $I^{\ge \mu(e)-\varepsilon}$ blocks $e$ and 
	$A^{\ge \mu(e)-\varepsilon}$ blocks $e$.
\end{lemma}
\begin{proof}
	$I^{\ge \mu(e)-\varepsilon}$ blocks $e$ because
	Lemma~\ref{lm:eps-opt-char}(2).
	$I^{\ge \mu(e)-\varepsilon}$ is an independent set of $A^{\ge \mu(e)-\varepsilon}$, so $A^{\ge \mu(e)-\varepsilon}$ blocks $e$ as well.
\end{proof}

Then we show that 
``is an $\varepsilon$-approximate subset of'' is a transitive relation.

\begin{lemma}\label{lm:chain-approx}
Let $A \subseteq B \subseteq C$. Suppose $A$ is an $\varepsilon_1$-approximate subset of $B$, and $B$ is an $\varepsilon_2$-approximate subset of $C$. Then $A$ is an $(\varepsilon_1+ \varepsilon_2)$-approximate subset of $C$.
\end{lemma}

\begin{proof}
Let $I \in \matroid_A$ be $\varepsilon_1$-optimal for $\matroid_B$. We prove it is $(\varepsilon_1+\varepsilon_2)$-optimal for $\matroid_C$.
For any element $e \in B \setminus A$, $I^{\ge \mu(e)-\varepsilon_1}$ blocks $e$. 
So $I^{\ge \mu(e) - (\varepsilon_1 + \varepsilon_2)}$ blocks $e$ as well. 
For $e \in C \setminus B$, we have $B^{\ge \mu(e)-\varepsilon_2}$ blocks $e$,
by Lemma~\ref{lm:approxsubset}. 
Set $r=\mu(e)-(\varepsilon_1+\varepsilon_2)$.
Using Lemma~\ref{lm:eps-opt-char}(3) with
$D_r=(B\setminus I)^{\ge \mu(e)-\varepsilon_2} \cup I^{\ge \mu(e)-(\varepsilon_1+\varepsilon_2)}$, we can see
that
$I^{\ge \mu(e)-(\varepsilon_1+\varepsilon_2)}$ is a basis in $\matroid_{D_r}$. 
Clearly $B^{\ge \mu(e)-\varepsilon_2} \subseteq D_r$.
So $D_r$ blocks $e$, which implies $I^{\ge \mu(e)-(\varepsilon_1+\varepsilon_2)}$ blocks $e$.
Hence, by Lemma~\ref{lm:eps-opt-char}, $I$ is $(\varepsilon_1+\varepsilon_2)$-optimal for $\matroid_C$.
\end{proof}

\subsection{Na\"ive Uniform Sampling Algorithm}

We start with a na\"ive uniform sampling algorithm, which samples each arm enough times 
to ensure that with high probability 
the empirical means are all within $\varepsilon/2$ from 
the true means,
and then outputs the optimal solution with respect 
to the empirical means.
The algorithm is a useful procedure in our final algorithm.

\begin{algorithm}[H]
\LinesNumbered
\setcounter{AlgoLine}{0}
\caption{\NAIVE($\matband,\varepsilon,\delta$)}
\label{algo:NAIVE-PROCEDURE}

	\KwData{A \matroidbanditstrong\ instance $\matband=(S,\matroid)$, with
	$\rank(\matroid) = k$, approximation error $\varepsilon$, confidence level $\delta$.}
	\KwResult{A basis $I$ in $\matroid$.}
        \smallskip
	$\hamean{} \leftarrow \UNIFORMSAMPLING(S,\varepsilon/2,\delta/|S|)$
	
	{\bf Return} The optimal solution $I$ with respect to the
        empirical means.
\end{algorithm}


\begin{lemma}\label{lm:NAIVE}
  The \NAIVE($\matband,\varepsilon,\delta$) algorithm outputs an
  $\varepsilon$-optimal solution for $\matband$ with probability at least
  $1-\delta$. The number of samples is $O(|S|\varepsilon^{-2} \cdot
  (\ln\delta^{-1} + \ln |S|))$.
\end{lemma}

\begin{proof}
  By Lemma~\ref{lm:UNIFORM-SAMPLE-PROCEDURE} and a simple union bound,
  we have $|\amean{e} - \hamean{e}| \le \varepsilon/2$ simultaneously for
  all arms $e \in S$ with probability $1-\delta$. Conditioning on that
  event, let $I$ be the returned basis.  For an arm $e \not\in I$, we
  have $I_{\hat\mu}^{\ge \hamean{e}}$ blocks $e$.  Note that for all
  arm $a \in I$, if $\hamean{a} \ge \hamean{e}$, we must have $\amean{a}
  \ge \amean{e}-\varepsilon$.  Hence, $I_{\hat\mu}^{\ge \hamean{e}}
  \subseteq I_{\mu}^{\ge \amean{e}-\varepsilon}$.  So $I_{\mu}^{\ge
    \amean{e}-\varepsilon}$ blocks $e$.  Then we have $I$ is
  $\varepsilon$-optimal by Lemma~\ref{lm:eps-opt-char}.  The sample
  complexity follows from the algorithm statement.
\end{proof}

\subsection{Sampling and Pruning}

Our optimal PAC algorithm 
applies the sampling and pruning 
technique, initially developed 
in the celebrated work of Karger, Klein and Tarjan~\cite{karger1995randomized}. 
They used the technique to obtain 
an expected linear-time algorithm for computing 
the minimum spanning tree.

We first describe the high level idea from
\cite{karger1995randomized}, which will be instructive for our later development.
Suppose we want to find the maximum spanning tree (MST).
We first construct a subgraph $F$
by sampling each edge with probability $p$; this subgraph may not be
connected, so we solve the maximum-weight spanning forest $I$ of $F$.
The key idea is this: we can use $I$ to prune a lot ``useless'' edges in
the original graph.  Formally, an edge $e=(u,v)$ is useless if edges
with larger cost in $I$ can connect $u$ and $v$: this is because the
cheapest edge in a cycle does not belong to the MST).  (In other words,
$e$ is useless if it is blocked by $I^{> \mu(e)}$.)  Having removed
these useless edges, we again recurse on the remaining graph, which
now has much fewer edges, to find the MST.
A crucial ingredient of the analysis in 
\cite{karger1995randomized} is to show 
that $I$ can indeed prune a lot of edges.


A proof from \cite{karger1995randomized, karger1998random} or 
\cite[pp.~299-300]{motwani2010randomized} shows that 
an optimal solution from a random subset 
can help us prune a substantial amount of elements.

\begin{lemma}\label{lm:rand-samp}(\cite[Lemma 2.1 and Remark 2.3]{karger1995randomized}) 
	Given a matroid $\matroid=(S,\indepfam)$ with an injective cost function $\mu: S \to \PR$,
	sample a subset $F$ of $S$ by selecting each element
	independently with probability $p$.
	An element $e \in S$ is called \emph{$F$-good} if $F^{>\mu(e)}$
	does not block $e$, else 
	it is \emph{$F$-bad}. 
	If the r.v.\ $X$ denotes the number of $F$-good elements in $S$,
	then 
	$X$ is stochastically dominated by $\NegBer(\rank(\matroid);p)$.
\end{lemma}

We also introduce a lemma which shows an $\varepsilon$-optimal solution $I$ in $F$ 
can be used to eliminate some sub-optimal arms.

\begin{lemma}\label{lm:F-eliminate}
	For a matroid $\matroid=(S,\indepfam)$ with cost function $\mu:S \to \PR$, 
	Let $F \subseteq S$ be a subset, 
	and $I$ be an $\alpha$-optimal basis for $\matroid_F$
	for some $\alpha>0$.
	If an element $e \in S \setminus I$ is $F$-bad, 
	$I^{\ge \mu(e)-\alpha}$ blocks $e$.
\end{lemma}
\begin{proof}
	As $e$ is $F$-bad, $F^{\ge \mu(e)}$ blocks $e$. 
	Let $r=\mu(e)-\alpha$, 
	and 
	$$
	D = (F \setminus I)^{\ge r+\alpha} \cup I^{\ge r} 
	= (F \setminus I)^{\ge \mu(e)} \cup I^{\ge \mu(e)-\alpha}.
	$$ 
	(In other words, we first add $\alpha$ to 
	the cost of every element in $I$, then consider all
	element with cost at least $\mu(e)$ in $F$).
	Then by Lemma~\ref{lm:char-opt}(4) and the fact 
	$I$ is $\alpha$-optimal for $\matroid_F$, 
	$I^{\ge \mu(e)-\alpha}$ is maximal for $\matroid_D$
	(in fact, it is optimal for $\matroid_D$ w.r.t.
	the modified cost function). 
	Clearly $F^{\ge \mu(e)} \subseteq D$, so $D$ blocks $e$ as well. Hence $I^{\ge \mu(e) - \alpha}$ also blocks $e$, 
	by Lemma~\ref{lm:blocks-1}.
\end{proof}

\subsection{Our Optimal PAC Algorithm}

Now, we present our algorithm for the PAC case, which is based on the
sampling-and-pruning technique discussed above.  Let
$p=\pvalue$. 
The algorithm runs as follows:
If the number of arms $|S|$ is sufficiently small, we simply run the
na\"{i}ve uniform sampling algorithm.  Otherwise, we sample a subset $F$
of $S$ by selecting each arm with probability $p$ independently, and
recurse on the sub-instance $\matband_F = (F,\matroid_F)$ to find an
$\alpha$-optimal solution $I$, where $\alpha=\varepsilon/3$.  Next, we
uniformly sample each arm in $S$ by calling
\UNIFORMSAMPLING($S,\lambda,\delta \cdot p / 8k$), where
$\lambda=\varepsilon/12$.  Then, we use $I$ to eliminate those
sub-optimal arms in $S\setminus I$.  More precisely, a sub-optimal arm
$e$ is blocked by the arms of $I$ with empirical values larger than
$\hat{\mu}_e-\alpha-2\lambda$. Finally, we invoke the algorithm
recursively on the remaining arms to find an $\alpha$-optimal solution,
which we output as the final result.  The pseudo-code can be found in
Algorithm~\ref{algo:RECURSIVE-PRUNING}. 

Note that \UNIFORMSAMPLING\ (step 6)
is the only step in which we take samples from the arms. 
Also note that in both recursive calls
we set the approximation error to be $\alpha=\varepsilon/3$.
Effectively, this makes sure that 
an arms surviving in deeper recursive call are sampled more times.
This feature is shared by other elimination-based method,
such as \cite{even2002pac,zhou2014optimal}. 
However, the way we choose which arms should be
eliminated is quite different.


\begin{algorithm}[t]

\LinesNumbered
\RestyleAlgo{boxed,ruled}
\SetAlgoVlined
\setcounter{AlgoLine}{0}
	\caption{\RECURPRUN($\matband,\varepsilon,\delta$)}
	\label{algo:RECURSIVE-PRUNING}
	\KwData{A \matroidbanditstrong\ instance $\matband=(S,\matroid)$, with
	$\rank(\matroid) = k$, approximation error $\varepsilon$, confidence level $\delta$.}
	\KwResult{A basis $I$ in $\matroid$.}
\smallskip	
	
	\uIf{$|S| \le 2p^{-2} \cdot \max(4 \cdot \ln 8\delta^{-1},k)$}{
		{\bf Return} \NAIVE($\matband,\varepsilon,\delta$)
	}
	
	Sample a subset $F \subseteq S$ by choosing each element with probability $p$ independently.
	
	$\alpha\leftarrow \varepsilon/3$, $\lambda\leftarrow \varepsilon/12$
	
	$I \leftarrow \RECURPRUN(\matband_F=(F,\matroid_F),\alpha,\delta/8)$ \label{line:rec-call-1}
	
	$\hamean{} \gets$ \UNIFORMSAMPLING($S,\lambda,\delta
        \cdot p / 8k$) \label{line:uni-sampl} 
	
	$S' \leftarrow I \cup \{ e \in S \setminus I \mid \text{ $I_{\hat\mu}^{\ge \hamean{e}-\alpha-2\lambda}$ does not block $e$}\}$ 
	
	{\bf Return} \RECURPRUN($\matband_{S'}=(S',\matroid_{S'}),\alpha
        ,\delta/4$) 
\end{algorithm}

\subsection{Analysis of the sample complexity}

In this subsection, we analyze \RECURPRUN\ and 
prove Theorem~\ref{theo:eps-opt-algo}.

\vspace{0.2cm}
\noindent
{\bf Theorem~\ref{theo:eps-opt-algo}	}
(rephrased)
{\em	
	Given a \matroidbanditstrong\ instance
	$\matband=(S,\matroid)$,
	Algorithm $\RECURPRUN\newline(\matband,\varepsilon,\delta)$\ 
	returns an $\varepsilon$-optimal solution,
	with probability at least $1-\delta$, 
	and uses at most
	$$
	O(n\varepsilon^{-2} \cdot (\ln k + \ln \delta^{-1})) 
	$$	
	samples. Here $k = \rank(\matroid)$, and $n = |S|$.
}

Let $c_1,c_2$ be two constants to be specified later. 
We will prove
by induction on $|S|$
that with probability at least $1-\delta$, \RECURPRUN($\matband=(S,\matroid),\varepsilon,\delta$) returns an $\varepsilon$-optimal solution,
using at most $c_1\cdot(|S|\varepsilon^{-2} \cdot (\ln k + \ln \delta^{-1}
+ c_2))$ samples.  Remember that $p = \pvalue$.

We first consider the simple case
where $|S|$ is not much larger than $k$.  
When $|S| \le 2p^{-2} \cdot \max(4 \cdot \ln 8\delta^{-1},k)$, 
we have that 
$\ln |S| = O(\ln \delta^{-1} + \ln k)$.
So the number of samples of \NAIVE\ is $O(|S|\varepsilon^{-2}\cdot(\ln |S|
+ \ln \delta^{-1})) = O(|S|\varepsilon^{-2}\cdot(\ln k + \ln
\delta^{-1}))$; 
by Lemma~\ref{lm:NAIVE}, the returned basis 
is $\varepsilon$-optimal with probability at least $1-\delta$. 
Hence the theorem holds in this case.

Now consider the case where $|S| > 2p^{-2} \cdot \max(4 \cdot
\ln 8\delta^{-1},k)$, and inductively assume that the theorem is
true for all instances of size smaller than $|S|$.
We first need the following lemma,
which describes the good events that 
happen with high probability.

\begin{lemma}\label{lm:good-F}
	Let $O$ be the unique optimal solution for $\matband=(S,\matroid)$. 
	With probability at least $1-\delta/2$, 
	the following statements hold simultaneously. 
	\begin{OneLiners}
		\item[1.] $|F| \le 2p \cdot |S|$\quad ($F$ is obtained in Line 3).
		\item[2.] There are at most $p \cdot |S|$ $F$-good elements in $S$.
		\item[3.] $|\mu_e-\hat{\mu}_e| \le \lambda$,
		for all elements $e\in O \cup I$ \quad ($I$ is obtained in Line 5).
		\item[4.] $I$ is an $\alpha$-optimal solution for $F$.
        \end{OneLiners}
\end{lemma}

\begin{proof}
	Let $n=|S|$. By Corollary~\ref{cor:B-upperbound}, we have that	
	\[
	\Pr[|F| > 2pn] =\Pr[\Bin(n,p) > 2pn]  \le e^{-pn/3} \le \delta/8,
	\]
	for $n \ge 8p^{-2}(\ln 8\delta^{-1})$. Moreover, let $X$ be the
        r.v.\ denoting the number of $F$-good elements in $S$. By
        Lemma~\ref{lm:rand-samp}, $X$ is dominated by $\NegBer(k;p)$,
        and hence
	\[
	\Pr[X>pn] \le \Pr[\NegBer(k;p) >pn] = \Pr[\Bin(pn,p) < k] \le \Pr[\Bin(pn,p) < \frac{1}{2}p^2n] \le e^{-\frac{1}{8}p^2n} \le \delta/8.
	\]
	The second inequality holds since $p^2n \ge 2k$, while the last
        inequality is due to $\frac{1}{8}p^2n \ge \ln\delta^{-1} +
        \ln8$.  In addition, by Lemma~\ref{lm:UNIFORM-SAMPLE-PROCEDURE}
        and a trivial union bound over all arms in $O \cup I$, the third
        statement holds with probability at least
        $1-(p\cdot\delta/8k)\cdot(2k) \ge 1-\delta/8$.  

        Finally, conditioning on the first statement, we have $|F| <
        |S|$, and hence by the induction hypothesis, with probability at
        least $1-\delta/8$, $I$ is an $\alpha$-optimal solution for $F$.
        Putting them together, all four statements hold with probability
        at least $1-\delta/8\cdot 4=1-\delta/2$.
\end{proof}

Now let $\event$ denote the event that all statements in Lemma~\ref{lm:good-F} hold. 
We show each $F$-bad element in $S \setminus I$ has a constant probability to be eliminated.

\begin{lemma}\label{lm:big-eliminate}
	Conditioning on $\event$, for an $F$-bad element $e \in S \setminus I$, $\Pr[e \in S'] \le \delta \cdot p/8k$.
\end{lemma}

\begin{proof}	
  Conditioning on $\event$, $I$ is $\alpha$-optimal for $F$. Hence, by
  Lemma~\ref{lm:F-eliminate}, for an $F$-bad element $e \in S \setminus
  I$, $I_{\mu}^{\ge \amean{e}-\alpha}$ blocks $e$.  By
    Lemma~\ref{lm:NAIVE}, $|\hamean{e} - \amean{e}| \le \lambda$ with
    probability $1-p\cdot\delta/8k$. Moreover, conditioning on $\event$,
    we have $|\hamean{a} - \amean{a}| \le \lambda$ for every element $
    a\in I$ (by Lemma~\ref{lm:good-F}.3).  Consequently, $I_{\mu}^{\ge
      \amean{e}-\alpha} \subseteq I^{\ge \hamean{e} -\alpha
      -2\lambda}_{\hat\mu}$, which implies that $I^{\ge \hamean{e}
      -\alpha -2\lambda}_{\hat\mu}$ blocks $e$.  By the definition
    of $S'$, this means $e \notin S'$.  Hence, we have $\Pr[e \in S'
    \mid \event] \le \delta \cdot p/8k$.
\end{proof}

Now, we show that with high probability, $S'$ is a
$2\varepsilon/3$-approximate subset of $S$, and the size of $S'$ is much
smaller than $|S|$.

\begin{lemma}\label{lm:good-S-1}
  Conditioning on $\event$, $|S'| \le 2p|S|$, and $S'$ is an
  $(\alpha+4\lambda)$-approximate subset of $S$, with probability $1-\delta/4$.
\end{lemma}

\begin{proof}
  Conditioned on event $\event$, there are at most $p\cdot|S|$ $F$-good
  elements in $S$ (by Lemma~\ref{lm:good-F}.2). If $X$ denotes the
  number of $F$-bad elements in $S \setminus I$ which remain in $S'$,
  Lemma~\ref{lm:big-eliminate} implies $\Ex[X] \le \delta\cdot (p/8k) \cdot
  |S \setminus I| \le \delta\cdot p/8 \cdot |S|$. By Markov's
  inequality, we have $\Pr[X \ge 0.5p|S|] \le \Pr[X \ge 4\cdot
  \delta^{-1}\Ex[X]] \le \delta/4$.  So there are at most $|I| +
  1.5p\cdot|S|\le k + 1.5p|S| \le 2p |S|$ elements in $S'$ with
  probability at least $1-\delta/4$.
	
  For the second part, observe that $O \cup S' \subseteq S$ is a
  $0$-approximate subset of $S$, so by Lemma~\ref{lm:chain-approx}, it
  suffices to show $S'$ is an $(\alpha+4\lambda)$-approximate subset for $O
  \cup S'$. Still conditioned on event $\event$, for all arms $e \in I \cup
  O$, we have $|\amean{e} - \hamean{e}| \le \lambda$.  So for an arm $e
  \in O \setminus S'$, we have $I^{\ge
    \hamean{e}-\alpha-2\lambda}_{\hat\mu}$ blocks $e$ (otherwise,
  $e$ should be included in $S'$), which implies $I_{\mu}^{\ge
    \amean{e}-\alpha-4\lambda}$ blocks $e$. Since $I \subseteq S'$, we
  can see $S'$ is an $(\alpha+4\lambda)$-approximate subset of $S'\cup O$
  by Definition~\ref{defi:eps-approx-subset}.
\end{proof}

Finally, we are ready to prove
Theorem~\ref{theo:eps-opt-algo}.


\begin{proofof}{Theorem~\ref{theo:eps-opt-algo}}
  Let $\event_G$ be the intersection of the event $\event$, the event
  that Lemma~\ref{lm:good-S-1} holds and the event that \RECURPRUN\
  (line 8) outputs correctly. Conditioning on event $\event$, $|S'| <
  |S|$, so by the induction hypothesis, the last event happens with
  probability at least $1-\delta/4$. Hence, $\Pr[\event_G] \ge
  1-\delta/2-\delta/4-\delta/4 = 1-\delta$.  We condition our following
  argument on $\event_G$.
	
  First we show the algorithm is correct.  By Lemma~\ref{lm:good-S-1},
  $S'$ is an $(\alpha+4\lambda)$-approximate subset of $S$, and the
  returned basis $J$ is an $\alpha$-optimal solution of $S'$, hence also
  an $\alpha$-approximate subset of $S'$. By the ``transitivity''
  property of Lemma~\ref{lm:chain-approx}, $J$ is an
  $(\alpha+\alpha+4\lambda)$-approximate subset of $S$. This is an
  $\varepsilon$-optimal solution of $S$ since
  $\alpha+\alpha+4\lambda=\varepsilon$.
	
  By Lemma~\ref{lm:good-F} and Lemma~\ref{lm:good-S}, we have $|F|\le
  2p\cdot|S|$ and $|S'| \le 2p\cdot|S|$.  By the induction hypothesis,
  the total number of samples in both recursive calls (line
  \ref{line:rec-call-1} and line 8) can be bounded by
  $$
  c_1\cdot4p |S|\varepsilon^{-2}(\ln \delta^{-1} + \ln k + c_2) \cdot 9 \le 36c_1p\cdot |S|\varepsilon^{-2}(\ln\delta^{-1} + \ln k + c_2).
  $$ 
  The number of samples incurred by \UNIFORMSAMPLING\ (line
  \ref{line:uni-sampl}) can be bounded by
  $$
  |S|\lambda^{-2} \cdot(\ln\delta^{-1} + \ln 16 + \ln p^{-1} + \ln k)/2 \le 72|S|\varepsilon^{-2} \cdot (\ln\delta^{-1} + \ln 16 + \ln p^{-1} + \ln k).
  $$ 
  Now, let $c_2 = \ln 16 +\ln p^{-1}$, which is a 
  constant.
  Then the total number of samples is bounded by
  \[
  (36p\cdot c_1 + 72)|S|\varepsilon^{-2} \cdot (\ln\delta^{-1} + \ln k + c_2).
  \]
  
  Setting $c_1=\max(120,c_0)$, and plugging in $p=0.01$, we can see the
  above quantity is bounded by $c_1\cdot|S|\varepsilon^{-2} \cdot (\ln
  \delta^{-1} + \ln k + c_2)$, which completes the proof.
\end{proofof}

\eat{
\begin{rem}
	Note that we don't have any guarantee with probability at most $\delta$. However, since it is a worst case algorithm, we can simply halt it if it does not terminate after taking a prescribed number of samples.
\end{rem}
}


%% file: pureexploration.tex

\newcommand{\substyle}{\mathrm}
\newcommand{\ro}{r_{\substyle{\substyle{elim}}}}
\newcommand{\rb}{r_{\substyle{\substyle{sele}}}}
\newcommand{\no}{n_{\substyle{\substyle{opt}}}}
\newcommand{\nb}{n_{\substyle{\substyle{bad}}}}
\newcommand{\mc}{\matroid_{\substyle{cur}}}
\newcommand{\scur}{S_{\substyle{cur}}}
\newcommand{\mbcur}{\matband_{\substyle{cur}}}
\newcommand{\snew}{S_{\substyle{new}}}

\newcommand{\Ans}{\mathsf{Ans}}
\newcommand{\OPTSOL}{\textsf{OPT}}
\newcommand{\BADARM}{\textsf{BAD}}
\newcommand{\optround}{elimination-round\xspace}
\newcommand{\badround}{selection-round\xspace}

\begin{algorithm}[H]
	\LinesNumbered
	\setcounter{AlgoLine}{0}
	\caption{\EXPGAPMATROID($\matband,\varepsilon,\delta$)}
	\label{algo:PURE-EXP-ALGO}
	\KwData{An \matroidbanditexact instance $\matband=(S,\matroid)$, with
		$\rank(\matroid) = k$, approx.\ error $\varepsilon$, confidence level $\delta$.}
	\KwResult{A basis $I$ in $\matroid$.}
	\smallskip	
	$\ro\leftarrow 1$, $\rb \leftarrow 1$
	
	\While{True}{
		$\scur \leftarrow \text{the arm set of }\mc$
		
		$\no \leftarrow \rank(\mc)$, $\nb \leftarrow |\scur| - \no$
		
		\uIf{$\no \le \nb$}{
			
			{\bf if }$\no = 0$ {\bf then break}
			
			$r \leftarrow \ro$
			
			$\varepsilon_r \leftarrow 2^{-r}/4$, $\delta_r \leftarrow \delta/100r^3$, $\ro \leftarrow \ro + 1$
			
			$I \leftarrow \RECURPRUN(\mbcur=(\scur,\mc),\varepsilon_r,\delta_r)$
			
			$\hamean{} \leftarrow \UNIFORMSAMPLING(I,\varepsilon_r/2,\delta_r/\no)$
			
			$\hamean{} \leftarrow \UNIFORMSAMPLING(\scur \setminus I, \varepsilon_r, \delta_r/\no)$
			
			$\snew \leftarrow I \cup \{e \in \scur \setminus I \mid I_{\hat\mu}^{\ge \hamean{e}+1.5\varepsilon_r} \text{ does not block $e$ in $\mc$}\}$
			
			$\mc \leftarrow \mc|\snew$
		}
		\uElse{
			\uIf{$\nb = 0$}{
				$\Ans \leftarrow \Ans \cup \scur$
				
				{\bf break}
			}
			$r \leftarrow \rb$
			
			$\varepsilon_r \leftarrow 2^{-r}/4$, $\delta_r \leftarrow \delta/100r^3$, $\rb \leftarrow \rb + 1$
			
			$\hamean{} \leftarrow \UNIFORMSAMPLING(\scur,\varepsilon_r,\delta_r/|\scur|)$
			
			$U \leftarrow \{e \in \scur\mid (\scur  \setminus \{e\} )_{\hat\mu}^{\ge \hamean{e} - 2\varepsilon_r} \text{ does not block $e$ in $\mc$} \}$
			
			$\Ans \leftarrow \Ans \cup U$
			
			$\mc \leftarrow \mc/U$
		}
	}
	{\bf Return} $\Ans$
\end{algorithm}

\section{An Algorithm for the \matroidbanditexact Problem}
\label{sec:matroid-pure-exploration}

We now turn to the \matroidbanditexact problem, and prove
Theorem~\ref{theo:pure-exploration}.  If we denote the unique optimal
basis by $\OPTSOL$, and let $\BADARM$ be the set of all other arms in $S
\setminus \OPTSOL$, our goal for the \matroidbanditexact problem is to
find this set $\OPTSOL$ with confidence $1 - \delta$ using as few
samples as possible.


Our algorithm \EXPGAPMATROID\
is based on our previous PAC result for \matroidbanditstrong,
and also borrow some idea from
the Exponential-Gap-Eliminating algorithm by \cite{karnin2013almost}.
It will run in rounds.  In each round, it
either tries to eliminate some arms in $\BADARM$ (we call such a round
an \emph{\optround}), or adds some arms from $\OPT$ into our solution
and removes them from further consideration (we call such a round a
\emph{\badround}). Let us give some details about these two kinds of
rounds.  Let $\mc$ be the current matroid defined over the remaining
arms, $\no$ be the number of remaining arms in $\OPT$, and $\nb$ be the
number of remaining arms in $\BADARM$. 
\begin{enumerate}
\item (\optround) When $\no \le \nb$, we are in an \optround.  In the
  $r^{th}$ \optround, first we find an $\varepsilon_r$-optimal
  solution $I$ for the current matroid $\mc$ by calling
  \RECURPRUN($\mc,\varepsilon_r,\delta_r$) (i.e., the PAC algorithm from
  Section~\ref{sec:matroid-PAC}) with $\varepsilon_r=2^{-r}/4$ and
  $\delta_r=\delta/100r^3$.  We sample each arm in $I$ by calling
  $\UNIFORMSAMPLING(I,\varepsilon_r/2,\delta_r/\no)$ to estimate their
  means.
  We do the same for arms $\scur \setminus I$ by calling
  $\UNIFORMSAMPLING(\scur \setminus I, \varepsilon_r, \delta_r/\no)$.
  Note the confidence parameter is not low enough to give accurate
  estimations for all arms in $\scur \setminus I$ with high probability:
  that would require us reducing the parameter to $\delta_r/|\scur
  \setminus I|$.  However, we will be satisfied with being accurate only
  for arms in $\OPT\setminus I$ with probability $\delta_r$. 

  Finally, we use $I$ to eliminate some sub-optimal arms.  In
  particular, an arm $e$ should be eliminated if $I_{\hat\mu}^{\ge
    \hamean{e}+1.5\varepsilon_r}$ blocks $e$, where $\hat{\mu}$ is the
  cost function defined by the empirical means obtained from the above
  \UNIFORMSAMPLING\ procedures.
\item (\badround) When $\nb < \no$, we are in a \badround.  In the
  $r^{th}$ \badround, we sample all the arms in $\mc$ by calling
  $\UNIFORMSAMPLING(\scur,\varepsilon_r,\delta_r/|\scur|)$.  We then
  select into our solution $\Ans$ those elements $e$ which are not
  blocked by all other elements in $\mc$ with larger empirical means,
  even if we slightly decrease $e$'s empirical mean by $2\varepsilon_r$.
  Having contracted these selected arms, we proceed to the next round.
\end{enumerate}
Finally, the algorithm terminates when either $\no = 0$ or $\nb =
0$. The pseudo-code is given as Algorithm~\ref{algo:PURE-EXP-ALGO}.



\subsection{Analysis of the algorithm}

Now, we prove the main theorem of this section 
by analyzing the correctness and sample complexity of \EXPGAPMATROID.

\noindent
{\bf Theorem~\ref{theo:pure-exploration}}
(rephrased)	
Given an \matroidbanditexact\ instance
$\matband(S,\matroid)$,
$\EXPGAPMATROID(\matband,\varepsilon,\delta)$ returns the optimal basis of $\matroid$,
with probability at least $1-\delta$, and uses
at most 
$$
O\left(\sum_{e \in S} \Delta_{e}^{-2}(\ln \delta^{-1}+\ln k 
+ \ln\ln \Delta_e^{-1}) \right)
$$
samples. Here, $k = \rank(\matroid)$ is the size of a basis of $\matroid$.

We first recall that the gap of an element $e$ (throughout this section, we only consider the cost function $\mu$ for gap) 
is defined to be 
$$
\Delta_{e}^{\matroid} := \begin{cases}
\OPT(\matroid) - \OPT(\matroid_{S \setminus \{e\}}),
\quad & e \in \OPT \text{ and } e \text{ is not isolated}; \\
+\infty, \quad & e \text{ is isolated;}\\
\OPT(\matroid) - \OPT(\matroid_{/\{e\}}) - \amean{e},
\quad &  e \not\in \OPT.
\end{cases}
$$

Note that we extend the definition to the isolated elements,
since the restrictions and contractions may result in 
such element (note that no loop is introduced during the process). 
We also need the following equivalent definition
(the equivalence follows from Lemma~\ref{lm:char-opt}), which may be convenient in some case:
\begin{align}
\label{eq:gap}
\Delta_{e}^{\matroid} := \begin{cases}
\max\{ w \in R \betw (S \setminus \{e\})^{> \amean{e}-w} \text{ does not block } e  \} & \text{ for } e \in \OPT;\\
\max\{  w \in R \betw S^{\ge \amean{e} + w} \text{ blocks } e  \} & \text{ for } e \not\in \OPT 
\end{cases}
\end{align}

First, we prove that our algorithm returns the optimal basis
with high probability.
In the following lemma,
We specify a few events on which 
we condition our later discussion.

\begin{lemma}\label{lm:good-S}
  With probability at least $1-\delta/5$, all of the following
  statements hold:
  \begin{OneLiners}
  \item[1.] In all \optround{s}, \RECURPRUN (line 9) returns correctly.
  \item[2.] In all \optround{s}, for all element $u \in I$, $|\amean{u} -
    \hamean{u}| < \varepsilon_r/2$.
  \item[3.] In all \optround{s}, for all element $u \in \OPT(\mc)$,
    $|\amean{u} - \hamean{u}| < \varepsilon_r$.
  \item[4.] In all \badround{s}, for all element $u \in \scur$, $|\amean{u} -
    \hamean{u}| < \varepsilon_r$.
  \end{OneLiners}
  We use $\event$ to denote the event that all above statements are
  true.
\end{lemma}

\begin{proof}
  In the $r^{th}$ \optround, the specification of the failure
  probabilities of \RECURPRUN and \UNIFORMSAMPLING imply the first
  three statements hold with probability $1-3\delta_r$. In the $r^{th}$
  \badround, the last statement holds with probability $1-\delta_r$.  A
  trivial union bound over all rounds gives
  \[
  \Pr[\lnot \event] \leq \sum_{r=1}^{+\infty} (3\delta_r + \delta_r) =
  \sum_{r=1}^{+\infty} 4\delta/100r^3 \le \delta/5,
  \]
  and the lemma follows immediately.
\end{proof}

\begin{lemma}\label{lm:algo-correct}
  Conditioning on $\event$, the subset $\Ans$ returned by the algorithm
  is the optimal basis $\OPT$.
\end{lemma}
\begin{proof}
  We condition on $\event$ in the following discussion.  We show that
  the algorithm only deletes arms from $\BADARM$ in every \optround, and
  it only selects arms from $\OPTSOL$ in every \badround. We say a round
  is correct if it satisfies the above requirements. We prove all rounds
  are correct by induction. Consider a round, and suppose all previous
  rounds are correct. Hence, at the beginning of the current round,
  $\Ans$ clearly is a subset of $\OPTSOL$, and $\OPT(\mc) = \OPTSOL
  \setminus \Ans$. There are two cases:

  If the current round is an \optround, consider an arm $u \in \OPT(\mc)
  = \OPTSOL \setminus \Ans$.  We can see that $I_{\mu}^{>\amean{u}}$
  does not block $u$ in $\mc$, by the characterization of the matroid
  optimal solutions in Lemma~\ref{lm:char-opt}.2.  Since $|\hamean{u} -
  \amean{u}| < \varepsilon_r$ for all $u\in \OPT(\mc)$, and $|\hamean{e}
  - \amean{e}| < \varepsilon_r/2$ for all $e \in I$, we also have
  $I_{\hat\mu}^{\ge \hamean{u}+1.5\varepsilon_r}$ does not block $u$ in
  $\mc$. Hence $u \in S_{new}$, and it is not eliminated.
	
  Next, consider a \badround and an arm $u \in \BADARM \cap \scur$.  By
  the induction hypothesis, in the beginning of the round, $u \not\in
  \OPT(\mc)$.  Then, we can see $u$ is blocked by $(\scur \setminus
  \{u\} )_{\mu}^{\ge \amean{u}}$ in $\mc$ by Lemma~\ref{lm:char-opt}.3.
  Again, for all arms $e$ in $\scur$, $|\amean{e}-\hamean{e}| <
  \varepsilon_r$, so $u$ is also blocked by $(\scur \setminus
  \{u\})_{\hat\mu}^{\ge \hamean{u}-2\varepsilon_r}$ in $\mc$.  Hence $u
  \not\in U$, and is not selected into $\Ans$.
	
  Finally, if the algorithms returns, we have $|\Ans| =|\OPTSOL|=
  \rank(\matroid)$. Since $\Ans \subseteq \OPTSOL$, it must be the
  case that $\Ans = \OPTSOL$.
\end{proof}

\subsubsection{Analysis of Sample Complexity}

To analyze the sample complexity, we need some additional notation.  Let
$\no^r$ (resp.\ $\nb^r$) denote $\no$ (resp.\ $\nb$) at the beginning of
$r^{th}$ \optround (resp.\ \badround).
Also, let $S_{\substyle{elim}}^{r}$
denote the arm set of $\mc$ at the end of the $r^{th}$ \optround, and
$S_{\substyle{sele}}^{r}$ denote the arm set of $\mc$ at the end of the $r^{th}$
\badround.  We partition the arms in $\OPT$ and $\BADARM$ based on their
gaps, as follows:
\[
\OPTSOL_{s} = \{ u \in \OPTSOL\mid 2^{-s} \le \Delta_u < 2^{-s+1}\},
\]
\[
\BADARM_{s} = \{ u \in \BADARM\mid 2^{-s} \le \Delta_u < 2^{-s+1}\}.
\]
Moreover, we define $\OPTSOL_{r,s} := S_{\substyle{sele}}^{r} \cap
\OPTSOL_{s}$, i.e., the set of arms in $\OPTSOL_s$ not selected in the
$r^{th}$ \badround---recall that in a \badround we aim to select those
arms into $\OPTSOL$. Similarly, define $\BADARM_{r,s} :=
S_{\substyle{elim}}^{r} \cap \BADARM_{s} $ as the set of arms in
$\BADARM_s$ not eliminated the $r^{th}$ \optround---again, in an
\optround we aim to delete those arms in $\BADARM$.

Very roughly speaking, the $s^{th}$ round 
is dedicated to deal with those arms with gap roughly
$O(2^{-s})$
(namely, an arm in $\OPT_s$ is likely to be 
selected in the $s^{th}$ \badround and 
an arm in $\BADARM_s$ is likely to be eliminated in the $s^{th}$
\optround
).
Now, we prove a crucial lemma, which states
that all elements in $\OPT_s$ should be selected  
in or before the $s^{th}$ \badround,
and the number of remaining elements in $\BADARM_s$
should drop exponentially after the $s^{th}$ \optround.

\begin{lemma}\label{lm:DEC-QUICK}
	Conditioning on event $\event$, with probability at least $1-4\delta/5$, we have 
	$$|\OPTSOL_{r,s}| = 0 
	\quad\text{and}\quad
	|\BADARM_{r,s}| \le \frac{1}{8} |\BADARM_{r-1,s}|
	\quad \text{for all } 
	1 \le s \le r.
	$$
\end{lemma}

Proving Lemma~\ref{lm:DEC-QUICK} requires some 
preparations.
All the following arguments are conditioned on event $\event$.
We first prove a useful lemma which roughly states that
if we select some elements in $\OPT$
and remove some elements in $\BADARM$,
the gap of the remaining instance
does not decrease.

\begin{lemma}\label{lm:GAP-INC}
	For two subsets $A,B$ of $S$ such that $A \subseteq \OPT \subseteq B$, consider the matroid $\matroid' = (\matroid|B)/A$. 
	Let $S'$ be its ground set.
	For all element $u \in S'$, 
	we have that 
	$\Delta_{u}^{\matroid'} \ge \Delta_{u}^{\matroid}$. 
\end{lemma}

\begin{proof}
	By the definition of matroid contraction and the fact that $A$ is independent, we can see for any subset
	$U \subseteq S'$, $U$ is independent in $\matroid'$ iff $U \cup A$ is independent in $\matroid$. 
	We also have $\rank_{\matroid'}(S') = \rank_{\matroid}(S) - |A|$ and $\OPT \setminus A$ is the unique optimal solution for $\matroid'$.
	
	Now, let $u \in S'$. 
	Suppose $u \in \OPT(\matroid') = \OPT \setminus A$. 
	\eat{
	We distinguish two cases: 
	\begin{enumerate}
	\item
	Suppose $\Delta_u^{\matroid} = +\infty$, i.e., $u$
	is isolated in $\matroid$.
	We need to show $u$ is also isolated in $\matroid'$.
	In fact, this can be seen from the submodularity of the
	rank function:
	\begin{align*}
	\rank_{\matroid'}(S')-\rank_{\matroid'}(S'\setminus\{e\})
	&=
	(\rank_{\matroid|B}(S')-|A|)-
	(\rank_{\matroid|B}(S'\setminus\{e\})-|A|) \\
	&
	=\rank_{\matroid}(S')-
	\rank_{\matroid}(S'\setminus\{e\})
	\geq \rank_{\matroid}(S)-
	\rank_{\matroid}(S\setminus\{e\})=1.
	\end{align*}
	}	
	Suppose for contradiction that 
	$\Delta_{u}^{\matroid'} < \Delta_{u}^{\matroid}$.
	Then we have a basis $I$ 
	contained in $S' \setminus \{u\}$ in $\matroid'$ such that 
	$$
	\mu(I) = \mu(\OPT \setminus A) - \Delta_{u}^{\matroid'}
	> \mu(\OPT \setminus A) - \Delta_{u}^{\matroid}.
	$$ 
	But this means $I \cup A$ is a basis 
	contained in $S \setminus \{u\}$ in $\matroid$ 
	such that 
	$\mu(I \cup A) > \OPT(\matroid) -\Delta_{u}^{\matroid}$, contradicting to the definition of $\Delta_{u}^{\matroid}$.
	Note that a non-isolated element in $\matroid$
	may become isolated in $\matroid'$ (for which $\Delta_{u}^{\matroid'}=+\infty$).

	Then, we consider the case 
	$u \not\in \OPT(\matroid') = \OPT \setminus A$.
	The argument is quite similar.
	Suppose for contradiction that
	$\Delta_{u}^{\matroid'} < \Delta_{u}^{\matroid}$.
	This means that 
	there exists a basis $I$ in $\matroid'$ such that 
	$u \in I$ and $\mu(I) > \OPT(\matroid') - \Delta_{u}^{\matroid} = \mu(\OPT \setminus A) - \Delta_u^{\matroid}$. 
	But this means $A \cup I$ is a basis in $\matroid$ such that 
	$\mu(A \cup I) > \OPT - \Delta_u^{\matroid}$. Since $u \in (A \cup I)$, 
	this contradicts the definition of $\Delta_u^{\matroid}$.
\end{proof}


\begin{proofof}{Lemma~\ref{lm:DEC-QUICK}}
	We first prove $|\OPTSOL_{r,s}| = 0$ for $r\geq s$.
	Suppose we are at the beginning 
	of the $r^{th}$ \badround.
	Let $A = \Ans$ and $B = \scur \cup \Ans$. 
	We can see $A \subseteq OPT \subseteq B$ and $\mc = (\matroid|B)_{/A}$.
	
	For any arm $u \in \OPTSOL_{r-1,s}$ such that $s\le r$, we have $\Delta_u \ge 2^{-s} \ge 2^{-r} \ge 4\varepsilon_r$. By Lemma~\ref{lm:GAP-INC}, we have $\Delta_{u}^{\mc} \ge \Delta_{u}^{\matroid} \ge 4\varepsilon_r$, which means $(\scur \setminus \{u\})_{\mu}^{> \amean{u} - 4\varepsilon_r}$ does
	not block $u$. Note that conditioning on $\event$, $|\amean{e} - \hamean{e}| < \varepsilon_r$ for all $e \in \scur$.
	This implies that 
	$(\scur \setminus \{u\})_{\hat\mu}^{\ge \hamean{u} - 2\varepsilon_r}$ does not block $u$ as well.
	So $u \in U$ ($U$ is defined in line 21) and consequently $|\OPTSOL_{r,s}| = 0$.
	
	Now, we prove the second part of the lemma.
	We claim that for $1 \le s \le r$, we have that
	\begin{align}
	\label{eq:probbadarm}
	\Pr[|\BADARM_{r,s}| \le \frac{1}{8} |\BADARM_{r-1,s}|] \ge 1-8\delta_r
	\end{align}
	The lemma follows directly from the claim
	by taking a union bound over all $s,r$ 
	such that $1\le s \le r$:
	$$
	\sum_{r=1}^{+\infty}\sum_{s=1}^{r} 8\delta_r = 
	\sum_{r=1}^{+\infty}8\delta/100r^2 \le 4\delta/5.
	$$
	What remains to prove is the claim (Inequality~\eqref{eq:probbadarm}).	
	Suppose we are at the beginning 
	of the $r^{th}$ \optround.
	Let $A = \Ans$ and $B = \scur \cup \Ans$. 
	Conditioning on event $\event$, we can see that
	$A \subseteq \OPT \subseteq B$ 
	and $\mc = (\matroid|B)/A$.
	
	For any arm $u \in \BADARM_{r-1,s}$ such that $s\le r$, we have $\Delta_u \ge 2^{-s} \ge 2^{-r} \ge 4\varepsilon_r$. By Lemma~\ref{lm:GAP-INC}, we have $\Delta_{u}^{\mc} \ge \Delta_{u}^{\matroid} \ge 4\varepsilon_r$. So $(\scur)_{\mu}^{\ge \amean{u} + 4\varepsilon_r}$ blocks $u$ in $\mc$ by the definition of $\Delta_u^{\mc}$. 
	As $I$ is $\varepsilon_r$-optimal for $\mc$, 
	we also have $u$ is blocked by 
	$I_{\mu}^{\ge \amean{u} + 3\varepsilon_r}$ in $\mc$. 
	This implies that $u \notin I$.
	
	Since we have $|\amean{u} - \hamean{u}| < \varepsilon_r$
	with probability $1-\delta_r$, 
	combining with the fact that
	$|\amean{e}-\hamean{e}| < \varepsilon_r/2$ for all $e\in I$
	(guaranteed by $\event$),
	$u$ is blocked by $I_{\hat\mu}^{> \hamean{u} + 1.5\varepsilon_r}$ with probability $1-\delta_r$.
	This implies that 
	$u \not\in \BADARM_{r,s}$ ($u$ should be eliminated in line 12). 
	From the above, we can see that
	\[
	\Ex[|\BADARM_{r,s}|] \le \delta_r |\BADARM_{r-1,s}|.
	\]
	By Markov inequality, 
	we have $\Pr[|\BADARM_{r,s}| \ge \frac{1}{8} |\BADARM_{r-1,s}|] \le 8\delta_r$, 
	which concludes the proof.
\end{proofof}

Finally, everything is in place to prove Theorem~\ref{theo:pure-exploration}.

\begin{proofof}{Theorem~\ref{theo:pure-exploration}}
	Let $\event_G$ be the intersection of event $\event$ and the event that Lemma~\ref{lm:DEC-QUICK} holds. By Lemma~\ref{lm:good-S} and Lemma~\ref{lm:DEC-QUICK}, $\Pr[\event_G] \ge 1-\delta$. 
	Now we condition on this event.
	
	The correctness has been proved in Lemma~\ref{lm:algo-correct}.
	We only need to bound the sample complexity of \EXPGAPMATROID.
	
	We first consider the number samples taken by 
	the \UNIFORMSAMPLING procedure.
	We handle the samples taken by $\RECURPRUN$ later.  
	Now, we bound the total number of samples taken from arms in $\OPTSOL_s$ in all \badround s.
	Notice that we can safely ignore all samples on arms in $\BADARM$
	since $\no \ge \nb$.
	By Lemma~\ref{lm:DEC-QUICK},  $|\OPTSOL_{r,s}|=0$ for $r\ge s$.
    So it can be bounded as:
	\begin{align*}
	O\left(\sum_{r=1}^{s} |\OPTSOL_{r-1,s}|\cdot(\ln \no + \ln\delta_r^{-1} )\varepsilon_r^{-2} \right) 
	\,\,\le&\,\, O\left(\sum_{r=1}^{s} |\OPTSOL_s|\cdot(\ln k + \ln\delta^{-1} + \ln r)\varepsilon_r^{-2} \right) \\
	\,\,\le&\,\, O\left(|\OPTSOL_s|\cdot(\ln k + \ln\delta^{-1} + \ln s)\cdot 4^{s} \right).
	\end{align*}
	
	Next, we consider the number of samples from \optround s.
	In an \optround, since $\no \le \nb$, 
	we only need to bound the number of samples 
	from $\BADARM$. 
	The total number of samples taken from arms 
	in $\BADARM_s$ in the first $s$ \optround{s} can be bounded as:
	\begin{align*}
	 O\left(\sum_{r=1}^{s} |\BADARM_{r-1,s}|\cdot(\ln |\scur| + \ln\delta_r^{-1} )\varepsilon_r^{-2} \right) 
	\,\,\le\,\, O\left(|\BADARM_s|\cdot(\ln k + \ln\delta^{-1} + \ln s)\cdot 4^{s} \right) 
	\end{align*}
	The inequality holds since $|\scur| \le \no + \nb \le 2\no \le 2k$
    in a \badround.
	Now, we bound the number of samples from the remaining rounds. 	
	Since $|\BADARM_{r,s}| \le \frac{1}{8} |\BADARM_{r-1,s}|$ when $r\ge s$, we have:
	\begin{align*}
	 O\left(\sum_{r=s+1}^{+\infty} |\BADARM_{r-1,s}|\cdot(\ln |\scur| + \ln\delta_r^{-1} )\varepsilon_r^{-2} \right) 
	&\,\,=\,\, O\left(\sum_{r=s+1}^{+\infty} \frac{1}{8^{r-s}} \cdot |\BADARM_{s}|\cdot(\ln k + \ln \delta^{-1} + \ln r )\varepsilon_r^{-2} \right) \\
	&\,\,=\,\, O\left(|\BADARM_{s}|\cdot(\ln k + \ln \delta^{-1} + \ln s )\cdot 4^{s}\right)
	\end{align*}
	Putting them together, we can see the number of 
	samples incurred by $\UNIFORMSAMPLING$ is bounded by:
	$$
	O\left(\sum_{s=1}^{+\infty} (|\BADARM_{s}| + |\OPTSOL_{s}|)\cdot(\ln k + \ln \delta^{-1} + \ln s )\cdot 4^{s}\right),
	$$	
	which simplifies to 	
	$
	O\left(\sum_{e \in S} \Delta_{e}^{-2}(\ln k + \ln \delta^{-1} + \ln\ln \Delta_e^{-1}) \right).
	$
	Finally, we consider the number of samples taken by
	$\RECURPRUN$.
	Noticing $\no = \rank(\mc)$, 
	the number of samples is
	$O(|\scur|(\ln \no + \ln \delta_r^{-1})\varepsilon_r^{-2})$. 
	So \RECURPRUN does not affect the sample complexity, and 
	we finish our proof.
\end{proofof}


%% file: epsmean.tex
\section{An Algorithm for \matroidbanditAvg}
\label{sec:matroid-eps-mean-PAC}

In this section we prove Theorem~\ref{theo:eps-mean-algo} by providing an algorithm with the desired sample complexity.

\subsection{Building Blocks}

\topic{1. Na\"ive Uniform Sampling Algorithm}

We first investigate how many samples the uniform sampling
algorithm needs in order to find an \EPSMEANOPT\ solution.
The uniform sampling procedure in \NAIVETWO\ has a 
different parameter from that in \NAIVE.

\begin{algorithm}[H]
\LinesNumbered
\setcounter{AlgoLine}{0}
	\caption{\NAIVETWO($\matband,\varepsilon,\delta$)}
	\label{algo:NAIVE-PROCEDURE2}
	\KwData{A \matroidbanditAvg\ instance $\matband=(S,\matroid)$, with
	$\rank(\matroid) = k$, approximation error $\varepsilon$, confidence level $\delta$.}
	\KwResult{A basis in $\matroid$.}	
	Sample each arm $e \in S$ for $Q_0 = 2\varepsilon^{-2}\cdot\ln\frac{2\binom{|S|}{k}}{\delta}/k$ times.
	 Let $\hamean{e}$ be its empirical mean.
	
	{\bf Return} The optimal solution $I$ with respect to the empirical means.
\end{algorithm}

The following lemma shows the performance of the above algorithm.
The proof is fairly standard and can be found in the appendix.

\begin{lemma}\label{lm:NAIVE2}
	Let $I$ be the output of $\NAIVETWO(\matband,\varepsilon,\delta)$. 
	With probability $1-\delta$, $I$ is \EPSMEANOPT\ for $\matband$.
	The total number of samples is at most
	$$O\left(|S| \cdot \Big(\ln\frac{|S|}{k}+\frac{\ln\delta^{-1}}{k}\Big)\varepsilon^{-2}\right).$$ 
\end{lemma}

\eat{
\begin{proofof}{Lemma~\ref{lm:NAIVE2}}
	First consider a basis $U$ in $\matroid$ (hence $|U|=k$). 
	We apply Proposition~\ref{prop:chernoff}.1 to all samples 
	taken from the arms in $U$:
	\begin{align*}
	\Pr\left[\Big|\frac{1}{k} \sum_{u \in U} \amean{u} - \frac{1}{k} \sum_{u \in U} \hamean{u} \Big|> \varepsilon/2\right]
	\le 2 \exp( - \varepsilon^2/2 \cdot Q_0 \cdot k) \le \delta/\binom{|S|}{k}.
	\end{align*}	
	Note that there are at most $\binom{|S|}{k}$ distinct bases. Hence, by a union bound over all bases, with probability $1-\delta$, 
	we have $\left|\frac{1}{k} \sum\nolimits_{u \in U} \amean{u} - \frac{1}{k} \sum\nolimits_{u \in U} \hamean{u}\right| \le \varepsilon/2$,
	for all basis $U$.
	
	Let $O = \OPT(\matroid) $. Then we have: $\frac{1}{k} \sum_{u \in I} \amean{u} \ge \frac{1}{k} \sum_{u \in I} \hamean{u} -\varepsilon/2 \ge \frac{1}{k} \sum_{u \in O} \hamean{u} -\varepsilon/2 \ge \frac{1}{k} \sum_{u \in O} \amean{u} -\varepsilon$,
	which means $I$ is \EPSMEANOPT\ for $\matband$.
	
	Finally, using the fact that $\binom{|S|}{k} \le \left(\frac{e|S|}{k}\right)^{k}$, the sample complexity can be 
	easily verified.
\end{proofof}
}

\topic{2. Elimination Procedure}

The following procedure \ELIMINATION\ is our main ingredient. 
Roughly speaking, it can help us to eliminate a constant fraction of 
the remaining arms while preserving the value of the optimal solution.

The idea of the procedure 
is  similar to \RECURPRUN\ in Section~\ref{sec:matroid-PAC}.
It runs as follows. 
It first samples a random subset $F$ by choosing each element with probability $p$ independently. Then it
finds a $\epsilon/5$-optimal solution $I$ for $F$
using \RECURPRUN. 
Afterwards,
we estimate the means of all arms in $I$ 
by calling $\UNIFORMSAMPLING(I,\alpha,\delta/6k)$.
This guarantees the additive errors for all arms in $I$ are at most $\alpha=\varepsilon/5$, with probability at least $1-\delta/6$.
Then we take a uniform number
$Q_0 = \beta^{-2}\max(\ln\frac{6}{\delta}/k,\ln200/2)$
of samples from the other arms. 
Next, we use $I$ to eliminate those sub-optimal arms in $S\setminus I$
based on their empirical means. 
More precisely, an arm $e$ should be eliminated
if it is blocked by the arms of $I$ with empirical means larger than 
the empirical mean of $e$ minus $3\varepsilon/5$.
The pseudo-code can be found in Algorithm~\ref{algo:ELIMINATION}.

\begin{algorithm}[H]
\LinesNumbered
\setcounter{AlgoLine}{0}
	\caption{\ELIMINATION($\matband,\varepsilon,\delta$)}
	\label{algo:ELIMINATION}
	\KwData{A \matroidbanditAvg\ instance $\matband=(S,\matroid)$, with
	$\rank(\matroid) = k$, approximation error $\varepsilon$, confidence level $\delta$.}
	\KwResult{A remained subset $S' \subseteq S$.}
	$\lambda\leftarrow \varepsilon/5 ,\alpha\leftarrow \varepsilon/5 ,\beta \leftarrow \varepsilon/5$
	
	Let $p \leftarrow \frac{100 \cdot (k + \ln\delta^{-1} + \ln 6)}{N}$
	
	Sample a random subset $F$ by choosing each element in $S$ with probability $p$ independently.
	
	$I \leftarrow \RECURPRUN(\matband_F=(F,\matroid_F),\lambda,\delta/6)$
	
	$\hamean{} \leftarrow \UNIFORMSAMPLING(I,\alpha,\delta/6k)$
	
	Sample each arm $e \in (S \setminus I)$ for $Q_0 = \beta^{-2}\max(\ln\frac{6}{\delta}/k,\ln200/2)$ times.
	Let $\hat{\mu}_e$ be its empirical mean.
	
	${\bf Return}\ S' = I \cup \{ e \in S \setminus I \mid I_{\hat\mu}^{\ge \hat{\mu}_e -\lambda-\alpha-\beta} \text{ does not block $e$} \}$
\end{algorithm}

We first specify the event we condition our analysis on.
The proof of the following lemma
is almost identical to that for Lemma~\ref{lm:good-F}
and can be found in the appendix.

\newcommand{\wO}{\widetilde{O}}

\begin{lemma}\label{lm:good-event1}
	With probability at least $1-2\delta/3$, the following statements hold:
	\begin{OneLiners}
		\item[1.] $|F| \le 2pN = O(k + \ln\delta^{-1})$.
		\item[2.] The number of $F$-good elements in $S$ is at most $0.04N$.
		\item[3.] $I$ is $\lambda$-optimal for $F$.
		\item[4.] $|\mu_e - \hat{\mu}_{e}| \le \alpha$ for all elements $e \in I$.
	\end{OneLiners}
\end{lemma}

\eat{
\begin{proofof}{Lemma~\ref{lm:good-event1}}
	We show that each claim happens with probability $\ge 1-\delta/6$.
	
	First, by Corollary~\ref{cor:B-upperbound}, we have $\Pr[|F| > 2pN] \le \exp(-pN/3) \le \delta/6$. 	
	Next, let the number of $F$-good elements in $S$ be $X$ and $A=0.04N$. 
	By Lemma~\ref{lm:rand-samp}, $X$ is stochastically dominated by $\NegBer(k;p)$.
	Then, by Lemma~\ref{lm:NB-bound} and $pA = \frac{1}{25} pN \ge 4k$, 
	we know $\Pr[X > A] \le \Pr[\NegBer(k;p) > A]=\Pr[\Bin(A,p) < k] \le \Pr[\Bin(A,p) < pA/4]$ . By Corollary~\ref{cor:B-upperbound},
	\[
	\Pr[\Bin(A,p) < pA/4] \le \exp(-9/16\cdot pA/2) \le \exp(-pA/4) \le \delta/6.
	\]
	So we have $\Pr[X > 0.04N] \le \delta/6$.
	
	Also, by Theorem~\ref{theo:eps-opt-algo}, $I$ is a
	$\lambda$-optimal solution for $F$ with probability $1-\delta/6$. Finally, by Lemma~\ref{lm:UNIFORM-SAMPLE-PROCEDURE} and a simple union bound, we have $|\amean{e} - \hamean{e}| \le \alpha$ for all elements $e \in I$ with probability $1-\delta/6$.
\end{proofof}
}

Denote the event described in the previous Lemma by $\event$. 
In the following we condition on event $\event$.
For notational convenience, we define the average value of a subset.

\begin{defi}
	Given a matroid $\matroid=(S,\indepfam)$ and $\mu : S \to \PR$ be a cost function, for any subset $A \subseteq S$, define 
	the average value of $A$ to be
	\[
	\OPTVAL_{\matroid,\mu}(A) :=
	\begin{cases}
	\OPT_{\mu}(\matroid_A)/\rank(A)  \quad &\quad \rank(A) = \rank(S),\\
	-\infty \quad &\quad \rank(A) < \rank(S).
	\end{cases}
	\]
	When $\matroid$ or $\mu$ is clear from the context, we omit it for convenience.
\end{defi}

The following lemma summarizes the properties of \ELIMINATION,
which roughly says that we can eliminate a significant portion of arms
while the value does not drop by much.

\begin{lemma}\label{lm:ELIMINATION-PROCEDURE}
	Let $N=|S|$. Suppose $N \ge 100 \cdot (k + \ln\delta^{-1} + \ln 6)$ and $\ELIMINATION(\matband,\varepsilon,\delta)$ returns $S'$. With probability at least $1-\delta$, the following statements hold:
	
	\begin{OneLiners}
		\item[1.] $\OPTVAL(S') \ge \OPTVAL(S) - \varepsilon$.
		
		\item[2.] $|S'| \le 0.1 |S|$.
		
		\item[3.] It takes $O\Big(\big((1+\ln\delta^{-1}/k)|S| + (k + \ln\delta^{-1})(\ln k + \ln \delta^{-1})\big)\varepsilon^{-2}\Big)$ samples.
	\end{OneLiners}
\end{lemma}


\begin{proof}
	Let $O$ be the optimal solution of $\matband$, and $\wO = O \setminus I$.
	We prove the first claim by showing 
	$\OPTVAL_{\mu}(I \cup (O \cap S')) \ge 
	\OPTVAL_{\mu}(O) -\varepsilon = \OPTVAL_{\mu}(S) - \varepsilon$. 
	Now, we fix the set $I$ found in step 4 (which satisfies $\event$).
	For each arm $u \in \wO$, set $\Delta_u = \min\{w \in \R \betw I_{\mu}^{\ge \amean{u}-w} \text{ blocks }u \}$.
	Note that for fixed $I$, $\Delta_u$ is a fixed number for each $u$. 
	As $u \in \wO \subseteq O$, we must have $\Delta_u \ge 0$. 
	Let $\gamma = \lambda+2\alpha+\beta$.
	
	For each $u \in O$, 
	let $
	\eta_u = \max(0,\Delta_u-\gamma)
	$ if $u \in \wO$, otherwise let $\eta_u = 0$.
	Note that $\eta_u$ is a fixed number, if we fix $I$.
	For each $u \in O$, 
	we also define the random variables $X_u$
	as follows: 
		$$
		X_u = \begin{cases}
			0   & \quad \text{ if } u\in S'; \\
			\eta_u   & \quad \text{ if } u\in O \setminus S'.
		\end{cases}
		$$
	Note that the randomness of $X_u$ is only due to step 6
	(which may cause $u\in S'$ or $u\notin S'$).
	We define $X = \frac{1}{k} \sum_{u \in O} X_u$. 
	
	Now, we show $\Pr[X > \beta] \le \delta/6$. 
	Note those random variables are independent.
	We first show that for each $u \in O$,
	$$
	\Pr[X_u = \eta_u]\leq \exp(-\eta_u^2\cdot 2 Q_0).
	$$ 
	It trivially holds for the case $\eta_u = 0$.
	So we only need to consider the case $\Delta_u > \gamma$.
	Suppose $u \not\in S'$.
	So, we have $I_{\hat\mu}^{\ge \hamean{u} - \lambda-\alpha-\beta}$ blocks $u$. Since $|\amean{e}-\hamean{e}| \le \alpha$ for all arms $e \in I$, we have $I_{\mu}^{\ge \hamean{u} - \lambda -2\alpha-\beta} = I_{\mu}^{\ge \hamean{u} - \gamma}$ blocks $u$ as well
	(recall $\gamma = \lambda+2\alpha+\beta$).
	By the definition of $\Delta_u$, we have $\hamean{u} - \gamma \le \amean{u}-\Delta_u$, which means $ \hamean{u} - \amean{u} < \gamma -\Delta_u = -\eta_u$. By Proposition~\ref{prop:chernoff}.1, we have $\Pr[\hamean{u} - \mu_{u} < -\eta_u] \le \exp(-\eta_u^2\cdot2Q_0)$.	
	So $\Pr[X_u = \eta_u] \le \Pr[\hamean{u} - \amean{u} < -\eta_u] \le \exp(-\eta_u^2\cdot2Q_0)$.
	
	Then we can apply Proposition~\ref{prop:zhouA4} with $t=2Q_0$ and obtain
	$$
	\Pr[X > \beta] < \exp(-\beta^2\cdot k Q_0) \le \delta/6.
	$$
	
	\newcommand{\wc}{\tilde{\mu}}
	
	Now, we show $X \le \beta$ implies $\OPTVAL_{\mu}(I\cup(O \cap S')) \ge \OPTVAL_{\mu}(O) -\varepsilon$. 
	We define a new cost function $\wc$ on $I \cup O$ as follows:
	$$
	\wc(e) = \begin{cases}
	\amean{e} &\quad e \in (I \cup O) \cap S'\\
	\amean{e}-\Delta_e &\quad e \in \wO \setminus S'. 
	\end{cases}
	$$
	Note that $I \subseteq S'$, 
	hence $\wc$ is well-defined for all elements in $I \cup O$. 
	For all element $u \in \wO \setminus S'$, by the definition of $\Delta_u$, $I_{\wc}^{\ge \wc(u)} = I_{\mu}^{\ge \amean{u}-\Delta_u}$ blocks $u$. 
	Hence, we have 
	$$
	\OPT_{\wc}(I \cup O) = \OPT_{\wc}((I \cup O)\cap S').
	$$
	By the definition of $\wc$, we can see that
	$$
	\OPT_{\mu}((I \cup O)\cap S') \ge 
	\OPT_{\wc}((I \cup O)\cap S') = \OPT_{\wc}(I \cup O) \ge \OPT_{\mu}(I \cup O) - \sum_{u \in \wO \setminus S'} \Delta_u.
	$$
	The first inequality is due to the fact that $\Delta_u > 0$ for all $u\in \wO$. 
	The last inequality is due to the fact that the total value we added is no more than 
	$\sum_{u \in \wO \setminus S'} \Delta_u$.
	Note that $\sum_{u \in \wO \setminus S'} \Delta_u \le \sum_{u \in O} (X_u + \gamma)$ by the definition of $X_u$.
	Hence $X \le \beta$ implies $\sum_{u \in \wO \setminus S'} \Delta_u \le (\beta + \gamma) \cdot k \le k\cdot \varepsilon$, which further implies $$
	\OPTVAL_{\mu}(S') \ge \OPTVAL_{\mu}(I \cup (O \cap S')) \ge \OPTVAL_{\mu}(I \cup O) - \varepsilon = \OPTVAL_{\mu}(S) -\varepsilon.
	$$
	This proves the first claim.
	
	Next, we prove the second claim.
	Let $N_B$ be the number of $F$-bad elements in $S\setminus I$. 
	Then by Lemma~\ref{lm:good-event1}
	and $N \ge 100 k$, 
	$ N_B\ge N-0.04N-k \ge 0.95N$. 
	Let $e$ be an $F$-bad element in $S \setminus I$. $F_{\mu}^{\ge \amean{e}}$ blocks $e$. Since $I$ is $\lambda$-optimal for $F$, we have $I_{\mu}^{\ge \amean{e}-\lambda}$ blocks $e$ by Lemma~\ref{lm:F-eliminate}.	
	By Proposition~\ref{prop:chernoff}.1, with probability at least $1 - 2\exp(-2Q_0\beta^2) \ge 1-\frac{1}{100}$, $|\amean{e} - \hamean{e}| \le \beta$. In that case, since $|\amean{u} - \hamean{u}| \le \alpha$ for all elements $u\in I$, 
	we have that 
	$I_{\mu}^{\ge \amean{e} - \lambda} \subseteq I_{\hat\mu}^{\ge \hamean{e} - \lambda - \alpha - \beta}$, which means $I_{\hat\mu}^{\ge \hamean{e} - \lambda - \alpha - \beta}$ blocks $e$.
	Hence, $e \not\in S'$. 
	Then we can see that for each $F$-bad element in $S \setminus I$, with probability 
	$\le \frac{1}{100}$, it is in $S'$. 
	All these events are independent. 
	So let $Y$ denote the number of the 
	remaining $F$-bad elements in $S \setminus I$.
	By the above argument, 
	we can see $Y$ is stochastically dominated by $\Bin(N_B,\frac{1}{100})$. 
	Then applying Corollary~\ref{cor:B-upperbound}, we can obtain that
	\begin{align*}
	\Pr[Y > 0.05N] &\le \Pr[\Bin(N_B,1/100) > 0.05N] \le \Pr[\Bin(N,1/100) > 0.05N] \\
	&\le \exp(-4^2\cdot 0.01N/3) \le \exp(-0.01N) \le \delta/6.
	\end{align*}	
	So with probability at least $1-2\delta/3-2\delta/6 = 1-\delta$, $Y \le 0.05N$, and $|S'| \le N - (N_B-Y) \le N-(0.95N-0.05N) \le 0.1|S|$,
	which proves the second claim.
	
	Finally, we examine the sample complexity.
	Since $|F| = O(k + \ln \delta^{-1})$,
	we can see that the procedure $\RECURPRUN(F,\lambda,\delta/6)$ takes $O((k + \ln\delta^{-1})(\ln k + \ln \delta^{-1})\varepsilon^{-2})$  samples by Theorem~\ref{theo:eps-opt-algo}. 
	By Lemma~\ref{lm:UNIFORM-SAMPLE-PROCEDURE}, $\UNIFORMSAMPLING(I,\alpha,\delta/6k)$ takes $O(k(\ln k + \ln\delta^{-1})\varepsilon^{-2})$  samples. 
	Finally, the sampling step (line 6)
	incurs $O((|S| + \ln\delta^{-1}/k)\varepsilon^{-2})$ samples. 
	Summing them together, we can see the sample complexity 
	of \ELIMINATION\ is as claimed.	
\end{proof}

\subsection{Main Algorithm}

In this section, we present our main algorithm for finding an \EPSMEANOPT\ solution.

The algorithm runs as follows: When $|S|$ is small enough, we just invoke \NAIVETWO\ and Lemma~\ref{lm:NAIVE2} can guarantee that we find 
an \EPSMEANOPT\ solution. 
Otherwise, we proceed in rounds. 
In the $r^{th}$ round, we invoke
\ELIMINATION($(S_r,\matroid_{S_{r}}),\varepsilon_r,\delta_r$),
where $S_r$ is the remaining set of arms,
$\varepsilon_r=\varepsilon/2^{r+1}$ and $\delta_r=\delta/2^{r+1}$, 
until the number of arms is smaller than a certain number. 
\ELIMINATION\ guarantees that the number of arms drops 
exponentially, hence the number of samples
is dominated by the first call to \ELIMINATION. 
In the end, we invoke \NAIVETWO\ on the remaining arms to find the final solution.
The pseudo-code can be found in Algorithm~\ref{algo:EPSMEANOPT-ALGO}.


\begin{algorithm}[H]
\LinesNumbered
\setcounter{AlgoLine}{0}

	\caption{\RECELIMI($\matband,\varepsilon,\delta$)}
	\label{algo:EPSMEANOPT-ALGO}

	\KwData{A \matroidbanditAvg\ instance $\matband=(S,\matroid)$, with
	$\rank(\matroid) = k$, approximation error $\varepsilon$, confidence level $\delta$.}
	\KwResult{A basis in $\matroid$.}
	\uIf{$|S|/k \le 10$ {\bf OR} $ \ln\delta^{-1} > k \ln \frac{|S|}{k}$ }{
		{\bf Return} \NAIVETWO($\matband,\varepsilon,\delta$)
	}
	$r \leftarrow 1$, $S_r \leftarrow S$
	
	\While{True}{
		$\delta_r \leftarrow \delta/2^{r+1}$, $\varepsilon_r \leftarrow \varepsilon/2^{r+1}$
		
		{\bf If}
		$|S_r| \le (\ln\delta_r^{-1} + k + \ln 6)(100 + \ln k + \ln \delta_r^{-1})$
		{\bf Break}
		
		$S_{r+1} \leftarrow \ELIMINATION(\matband_{S_r}=(S_r,\matroid_{S_{r}}),\varepsilon_r,\delta_r)$
		
		$r \leftarrow r+1$
	}
	
	{\bf Return} $I=$\NAIVETWO($\matband_{S_r}=(S_r,\matroid_{S_{r}}),\varepsilon/2,\delta/2$)
\end{algorithm}

\vspace{0.4cm}

Now, we prove Theorem~\ref{theo:eps-mean-algo}.

\noindent
{\bf Theorem~\ref{theo:eps-mean-algo}} (rephrased)
{\em	
	Given a \matroidbanditAvg\ instance $\matband=(S,\matroid)$,
	\RECELIMI($\matband,\varepsilon,\delta$)
	returns an \EPSMEANOPT\ solution,
	with probability at least $1-\delta$, and its sample complexity is at most
	\[
	O\Big( \big(n\cdot(1 + \ln\delta^{-1}/k) + (\ln\delta^{-1}+k)(\ln k\ln\ln k + \ln\delta^{-1}\ln\ln\delta^{-1} )\big)\varepsilon^{-2} \Big),
	\]
	
	in which $n = |S|$ and $k=\rank(\matroid)$.
}


\begin{proof}
	Note that when $\frac{n}{k} \le 10$ or $\ln\delta^{-1} > k \ln \frac{n}{k}$, $\NAIVETWO(\matband,\varepsilon,\delta)$ 
	returns a correct solution and 
	takes $O(n(1+\ln\delta^{-1}/k))$ samples. So from now on we assume $\ln\delta^{-1} \le k \ln \frac{n}{k}$ and $n > 10k$.
	
	Note that by a simple union bound, with probability at least $1-\delta$, all calls to \ELIMINATION\ return correctly, and the last call to \NAIVETWO\ also returns correctly. Denote this  event as $\event$ and the following argument is conditioned on $\event$.
	
	First we show the correctness. Suppose there are $t$ rounds in total. We can see that
	$$
	\OPTVAL(S_t) \ge \OPTVAL(S_1) - \sum_{i=1}^{t-1} \varepsilon/2^{i+1} \ge \OPTVAL(S) - \varepsilon/2.
	$$
	Clearly, $\OPTVAL(I) \ge \OPTVAL(S_t) -\varepsilon/2$.
	Hence $\OPTVAL(I) \ge \OPTVAL(S) -\varepsilon$, which means $I$ is \EPSMEANOPT\ for $\matband$.
	
	Now, we bound the sample complexity.
	\ELIMINATION\ is only called during the first $t-1$ rounds, and in which we have $|S_r| > (\ln\delta_r^{-1} + k + \ln 6)(100 + \ln k + \ln \delta_r^{-1}) \ge 100(\ln\delta_r^{-1} + k + \ln 6)$. \
	By Lemma~\ref{lm:ELIMINATION-PROCEDURE}, the sample complexity 
	for round $r$
	is $O\Big(\big((1+\ln\delta_r^{-1}/k)|S_r| + (k + \ln\delta_r^{-1})(\ln k + \ln \delta_r^{-1})\big)\varepsilon_r^{-2}\Big) = O\Big((1+\ln\delta_r^{-1}/k)|S_r|\varepsilon_r^{-2}\Big)$. 
	
	Since $|S_{r}| \le |S_1| \cdot 0.1^{r-1}$ for $1\le r \le t$ and $\varepsilon_r=\varepsilon/2^{r+1}$, we can bound the total number of samples of the first $t-1$ rounds by
	\[
	O\Big(\sum\nolimits_{r=1}^{t-1} 4^{r+1} \cdot (1 + (\ln \delta^{-1} + r)/k) \cdot n \cdot 0.1^{r-1} \varepsilon^{-2} \Big) = O\Big( n(1 + \ln\delta^{-1}/k)\varepsilon^{-2} \Big).
	\]
	
	Finally, consider the last round $t$.
	We have $|S_t| \le (\ln\delta_t^{-1} + k + \ln 6)(100 + \ln k + \ln \delta_t^{-1}) = O\big((\ln\delta^{-1} + t + k)(\ln k + \ln \delta^{-1} + t)\big)$. Since $S_t \ge k$, 
	we can see that $t \le O(\ln \frac{n}{k})$.
	Now we distinguish two cases: 
	
	\begin{enumerate}
	\item
	$n \le k^3$: In this case, $t = O(\ln k)$ and $|S_t| \le O\big((\ln\delta^{-1} + k)(\ln k + \ln \delta^{-1})\big)$. By Lemma~\ref{lm:NAIVE2}, the sample complexity for \NAIVETWO($\matband_{S_t},\varepsilon/2,\delta/2$)\ can be bounded by
	\begin{align*}
	& O\Big(\big((\ln\delta^{-1} + k)(\ln k + \ln \delta^{-1}) \cdot \ln \frac{(\ln\delta^{-1} + k)(\ln k + \ln \delta^{-1})}{k} + |S_t|/k \cdot \ln\delta^{-1} \big)\varepsilon^{-2}\Big)\\
	\le&
	O\Big(\big((\ln\delta^{-1} + k)(\ln k + \ln \delta^{-1}) \cdot (\ln\ln k + \ln\ln \delta^{-1}) + |S_t|/k \cdot \ln\delta^{-1} \big)\varepsilon^{-2}\Big)\\
	\le&
	O\Big(\big((\ln\delta^{-1} + k)(\ln k \ln\ln k + \ln \delta^{-1}\ln\ln \delta^{-1}) + n/k \cdot \ln\delta^{-1} \big)\varepsilon^{-2}\Big).
	\end{align*}
	\item
	$n \ge k^3$: In this case, $t \le O(\ln n)$. 
	Recalling that $\ln \delta^{-1} \le k \ln \frac{|S_t|}{k} \le O(n^{0.4})$. The sample complexity for \NAIVETWO($\matband_{S_t},\varepsilon/2,\delta/2$) can be bounded by
	$$
	O\Big(\big(n^{0.4} \cdot n^{0.4} \cdot \ln n + |S_t|/k \cdot \ln\delta^{-1} \big)\varepsilon^{-2}\Big) \le 
	O\Big(\big(n + n/k \cdot \ln\delta^{-1} \big)\varepsilon^{-2}\Big).
	$$
	\end{enumerate}
	Putting them together, we can see the total sample complexity is
	\[
	O\Big( \big(n(1 + \ln\delta^{-1}/k) + (\ln\delta^{-1}+k)(\ln k\ln\ln k + \ln\delta^{-1}\ln\ln\delta^{-1} )\big)\varepsilon^{-2} \Big).
	\]
	This completes the proof of the theorem.
\end{proof}


%% file: appendix.tex
\section{Preliminaries in Probability}

    We first introduce the following versions of the standard Chernoff-Hoeffding bounds.
    
    \begin{prop}\label{prop:chernoff}
    	Let $X_i (1\le i \le n)$ be $n$ independent random variables with values in $[0,1]$. Let $X = \frac{1}{n}\sum_{i=1}^{n} X_i$. The following statements hold:
    	
    	\begin{enumerate}
    		\item For every $t > 0$, we have that
    		\[
    		\Pr[X - \Ex[X] \ge t] \le \exp(-2t^2n),\ and
    		\]
    		\[
    		\Pr[X - \Ex[X] \le -t] \le \exp(-2t^2n).
    		\]
    		\item For any $\epsilon > 0$, we have that
    		\[
    		\Pr[X < (1-\epsilon)\Ex[X]] \le \exp(-\epsilon^2n\Ex[X]/2),\ and
    		\]
    		\[
    		\Pr[X > (1+\epsilon)\Ex[X]] \le \exp(-\epsilon^2n\Ex[X]/3).
    		\]
    	\end{enumerate}
    \end{prop}
    
    Applying the above Proposition, we can get useful upper bound for the binomial distribution.
    \begin{cor}\label{cor:B-upperbound}
    	Suppose the random variable $X$ follows the binomial
    	distribution $\Bin(n,p)$,
    	i.e., $\Pr[X=k]={n\choose k}p^k(1-p)^{n-k}$ for $k\in \{0,1,\ldots,n\}$.
    	It holds that for any $\epsilon>0$,	
    	\[
    	\Pr[X < (1-\epsilon)pn] \le \exp(-\epsilon^2pn/2),\ and
    	\]
    	\[
    	\Pr[X > (1+\epsilon)pn] \le \exp(-\epsilon^2pn/3).
    	\]
    \end{cor}
    
    We also need the following Chernoff-type concentration inequality (see Proposition A.4. in \cite{zhou2014optimal}).
    
    \begin{prop}\label{prop:zhouA4}
    	Let $X_i(1 \le i \le k)$ be independent random variables. Each $X_i$ takes value $a_i$ $(a_i \ge 0)$
    	with probability at most $\exp(-a_i^2t)$ for some $t \ge 0$, and 0 otherwise. Let $X = \frac{1}{k}\sum_{i=1}^{k} X_i$. For every $\epsilon>0$,
    	when $t \ge \frac{2}{\epsilon^2}$, we have that $$ \Pr[X > \epsilon] < \exp(-\epsilon^2tk/2).  $$
    \end{prop}

We need to introduce the definition of the negative binomial distributions (see e.g., \cite[pp.446]{motwani2010randomized}).

\begin{defi}
	Let $X_1,X_2,\dotsc,X_n$ be i.i.d. random variables
	with the common distribution being the geometric distribution with parameter $p$. 
	The random variable $X = X_1 + X_2 + \dotsc + X_n$ denotes the number of coin flips (each one has probability $p$ to be HEAD) needed to obtain $n$ HEADS. The random variable $X$ has the negative binomial distribution with parameters $n$ and $p$,  denote as 
	$X \sim \NegBer(n;p)$.
\end{defi}


\begin{lemma}\label{lm:NB-bound}
	$\Pr[\NegBer(n;p) > r] =\Pr[\Bin(r;p) < n]$.
\end{lemma}
\begin{proof}
	Consider the event $\NegBer(n;p) > r$.
	By the definition of $\NegBer(n;p)$, 
	it is equivalent to the event that during 
	the first $r$ coin flips, 
	there are less than $n$ HEADS. 
	The lemma follows immediately.
\end{proof}

\begin{defi} (stochastic dominance)
We say a random variable $X$ {\em stochastically dominates} another random variable $Y$ 
if for all $r \in \R$, we have $\Pr[X > r ] \ge \Pr[Y > r]$.
\end{defi}

\section{Missing Proofs}
	\label{app:missingpf}
	
	
	\begin{proofof}{Proposition~\ref{prop:epsoptimality}} 
		Let $I$ be an $\varepsilon$-optimal solution. We show it is also elementwise-$\varepsilon$-optimal.
		Let $o_i$ be the arm with the $i^{th}$ largest mean in $\OPT$
		and $a_i$ be the arm with the $i^{th}$ largest mean in $I$.
		Suppose for contradiction that $\mu(a_i)< \mu(o_i)-\varepsilon$
		for some $i\in [k]$ where $k=\rank(\matband)$. 
		Now, consider the sorted list of the arms according to
		the modified cost function
		$\newcostfunc{I}{\varepsilon}$.
		The arm $a_i$ is ranked after $o_i$ and all $o_j$ with $j<i$.
		Let $P$ be the set of all arms with mean no less than $o_i$ with respect to $\newcostfunc{I}{\varepsilon}$.
		Clearly, $\rank(P)\geq i$.
		So the greedy algorithm should select at least $i$ 
		elements in $P$, while $I$
		only has at most $i-1$ elements in $P$, contradicting 
		the optimality of $I$ with respect to $\newcostfunc{I}{\varepsilon}$. 		
		
		For the second part,
		take a \bestkarm ($k=2$) instance with four arms: 
		$\mu(a_1)=0.91, \mu(a_2)=0.9, \mu(a_3)=0.89, \mu(a_4)=0.875$.
		The set $\{a_3, a_4\}$ is elementwise-$0.3$-optimal, 
		but not $0.3$-optimal. 
	\end{proofof}

		\begin{proofof}{Lemma~\ref{lm:NAIVE2}}
			First consider a basis $U$ in $\matroid$ (hence $|U|=k$). 
			We apply Proposition~\ref{prop:chernoff}.1 to all samples 
			taken from the arms in $U$:
			\begin{align*}
			\Pr\left[\Big|\frac{1}{k} \sum_{u \in U} \amean{u} - \frac{1}{k} \sum_{u \in U} \hamean{u} \Big|> \epsilon/2\right]
			\le 2 \exp( - \epsilon^2/2 \cdot Q_0 \cdot k) \le \delta/\binom{|S|}{k}.
			\end{align*}	
			Note that there are at most $\binom{|S|}{k}$ distinct bases. Hence, by a union bound over all bases, with probability $1-\delta$, 
			we have $\left|\frac{1}{k} \sum\nolimits_{u \in U} \amean{u} - \frac{1}{k} \sum\nolimits_{u \in U} \hamean{u}\right| \le \epsilon/2$,
			for all basis $U$.
			
			Let $O = \OPT(\matroid) $. Then we have: $\frac{1}{k} \sum_{u \in I} \amean{u} \ge \frac{1}{k} \sum_{u \in I} \hamean{u} -\epsilon/2 \ge \frac{1}{k} \sum_{u \in O} \hamean{u} -\epsilon/2 \ge \frac{1}{k} \sum_{u \in O} \amean{u} -\epsilon$,
			which means $I$ is \EPSMEANOPT\ for $\matband$.
			
			Finally, using the fact that $\binom{|S|}{k} \le \left(\frac{e|S|}{k}\right)^{k}$, the sample complexity can be 
			easily verified.
		\end{proofof}

	\begin{proofof}{Lemma~\ref{lm:good-event1}}
		We show that each claim happens with probability $\ge 1-\delta/6$.
		
		First, by Corollary~\ref{cor:B-upperbound}, we have $\Pr[|F| > 2pN] \le \exp(-pN/3) \le \delta/6$. 	
		Next, let the number of $F$-good elements in $S$ be $X$ and $A=0.04N$. 
		By Lemma~\ref{lm:rand-samp}, $X$ is stochastically dominated by $\NegBer(k;p)$.
		Then, by Lemma~\ref{lm:NB-bound} and $pA = \frac{1}{25} pN \ge 4k$, 
		we know $\Pr[X > A] \le \Pr[\NegBer(k;p) > A]=\Pr[\Bin(A,p) < k] \le \Pr[\Bin(A,p) < pA/4]$ . By Corollary~\ref{cor:B-upperbound},
		\[
		\Pr[\Bin(A,p) < pA/4] \le \exp(-9/16\cdot pA/2) \le \exp(-pA/4) \le \delta/6.
		\]
		So we have $\Pr[X > 0.04N] \le \delta/6$.
		
		Also, by Theorem~\ref{theo:eps-opt-algo}, $I$ is a
		$\lambda$-optimal solution for $F$ with probability $1-\delta/6$. Finally, by Lemma~\ref{lm:UNIFORM-SAMPLE-PROCEDURE} and a simple union bound, we have $|\amean{e} - \hamean{e}| \le \alpha$ for all elements $e \in I$ with probability $1-\delta/6$.
	\end{proofof}